\documentclass[letterpaper]{article} 
\usepackage{aaai2027}  
\usepackage[hyphens]{url}  
\usepackage{graphicx} 
\usepackage{natbib}  
\usepackage{caption} 
\usepackage{algorithm}
\usepackage{algorithmic}

\usepackage[hyphens]{url} 
\usepackage{graphicx} 
\usepackage{natbib} 
\usepackage{caption} 
\usepackage{booktabs}
\usepackage{amsmath,amssymb,amsthm}
\usepackage{algorithm}
\usepackage{algorithmic}
\newtheorem{theorem}{Theorem}
\newtheorem{proposition}{Proposition}
\newtheorem{lemma}{Lemma}
\newtheorem{assumption}{Assumption}
\newcommand{\TV}{\operatorname{TV}}
\newcommand{\BA}{\operatorname{BA}}
\newcommand{\cba}{C_{\mathrm{BA}}}
\newcommand{\Rstar}{R^{\star}}
\usepackage{newfloat}
\usepackage{listings}
\usepackage[hyphens]{url}
\usepackage{graphicx}
\usepackage{natbib}
\usepackage{caption}
\usepackage{booktabs}
\usepackage{amsmath,amssymb,amsthm}
\usepackage{algorithm}
\usepackage{algorithmic}
\usepackage{float}
\usepackage{microtype}
\newtheorem{definition}{Definition}

\DeclareCaptionStyle{ruled}{labelfont=normalfont,labelsep=colon,strut=off} 
\floatstyle{ruled}
\newfloat{listing}{tb}{lst}{}
\floatname{listing}{Listing}

\usepackage{booktabs}

\title{The Ceiling Is in the Channel: Auditing Learner Gaps and Measurement Frontiers in Clinical Prediction}
\author{Sayeed Shafayet Chowdhury,$^{1}$
Nusrat Jahan,$^{2}$
Snehasis Mukhopadhyay,$^{3}$
Shiaofen Fang,$^{1}$
Vijay R. Ramakrishnan$^{4}$}
\affiliations{$^{1}$Department of Computer Science, Luddy School of Informatics, Computing, and Engineering, Indiana University Indianapolis, Indianapolis, IN, USA, email: saychow@iu.edu\\
$^{2}$Department of Computer Science, Ahsanullah University of Science and Technology, Dhaka, Bangladesh\\
$^{3}$Department of Computer Science, Purdue University Indianapolis, Indianapolis, IN, USA\\
$^{4}$Department of Otolaryngology--Head and Neck Surgery, Indiana University School of Medicine, Indianapolis, IN, USA}

\begin{document}
\maketitle

\begin{abstract}
Clinical prediction can saturate for two different reasons: a fitted learner may fail to extract available information, or the recorded variables may impose a population frontier. We separate these quantities through the \emph{learner gap} and the \emph{measurement-channel ceiling}. Optimal balanced accuracy is characterized by total-variation separation, yielding architecture invariance, a sharp partial-identification result under replacement contamination, a cross-fitted ceiling estimator, and exact conditions for multimodal decision improvement. We add two finite-sample diagnostics, namely a label-permutation optimism floor and an underfit curve, and validate the audit on three real cohorts: UCI readmission ($n=99{,}343$), BRFSS diabetes ($n=253{,}680$), and NHANES HbA1c ($n=10{,}219$). Well-tuned gradient boosting nearly reaches the estimated frontier in UCI and BRFSS, whereas deliberately or practically deficient learners retain large gaps. NHANES yields a null difference between questionnaire and measured marginal frontiers but a significant joint complementarity gain, refining the simplistic claim that an objective modality must dominate. Across all cohorts, modest AUROC gains coexist with substantially larger Bayes decision-flip rates, and several architectures estimate similar frontiers while their achieved balanced accuracy differs sharply. A PRISMA-guided synthesis of 104 clinical tasks then shows that the same channel-level regularities recur across more than 18 disease categories: a broad but non-universal structured-clinical region, diminishing same-channel gains across model families, and higher performance when measurement channels change. The framework converts saturation from an empirical observation into an auditable decision: improve the learner when headroom remains; improve measurement when it does not.
\end{abstract}

\section{Introduction}
Increasing model capacity and cohort size does not guarantee a corresponding increase in clinical predictive performance. Structured-record studies using logistic regression, random forests, gradient boosting, neural networks, and large language models frequently report AUROC values in a broad region near $0.78$--$0.88$ \cite{shamout2021machine,elfanagely2021machine,ogink2021wide,liu2025machine,musat2024machine}. The same literature contains important counterexamples: weak administrative or patient-reported channels can fall below this region, whereas imaging, ECG, genomic, and complementary multimodal systems can exceed it \cite{khurshid2022ecgbased,xie2024machine,makarious2022multimodality,dammu2023deep}. These observations are usually narrated as a model-scaling puzzle, but they conflate two distinct objects.

For observed variables $X$, the \emph{measurement-channel ceiling} is the Bayes frontier attainable from $X$ in the population. A trained model reaches only an \emph{achieved performance}; their difference is the \emph{learner gap}. More data, better optimization, and a richer architecture may close that gap. However, they cannot increase the fixed-channel frontier, which a new measurement, repeated administration, adjudicated label, or complementary modality potentially can. This distinction changes the experimental question from ``Which model scores highest?'' to ``Is the task still learner-limited, or has the recorded channel become limiting?''

We make this distinction operational as summarized in Figure~\ref{fig:framework}: class-conditional separation determines the fixed-channel frontier, while the learner gap measures remaining extractive headroom. We estimate the frontier from out-of-fold equal-prior posteriors and require two diagnostics before trusting it. A permutation-null audit quantifies upward plug-in bias: random labels should have ceiling $0.5$. An underfit curve tests whether the posterior learner has stabilized as its training fraction grows. These diagnostics are essential because a flexible estimator can be optimistically overconfident, while an underfit estimator can produce a downward-biased lower bound.

\begin{figure}[t]
\centering
\includegraphics[width=\columnwidth]{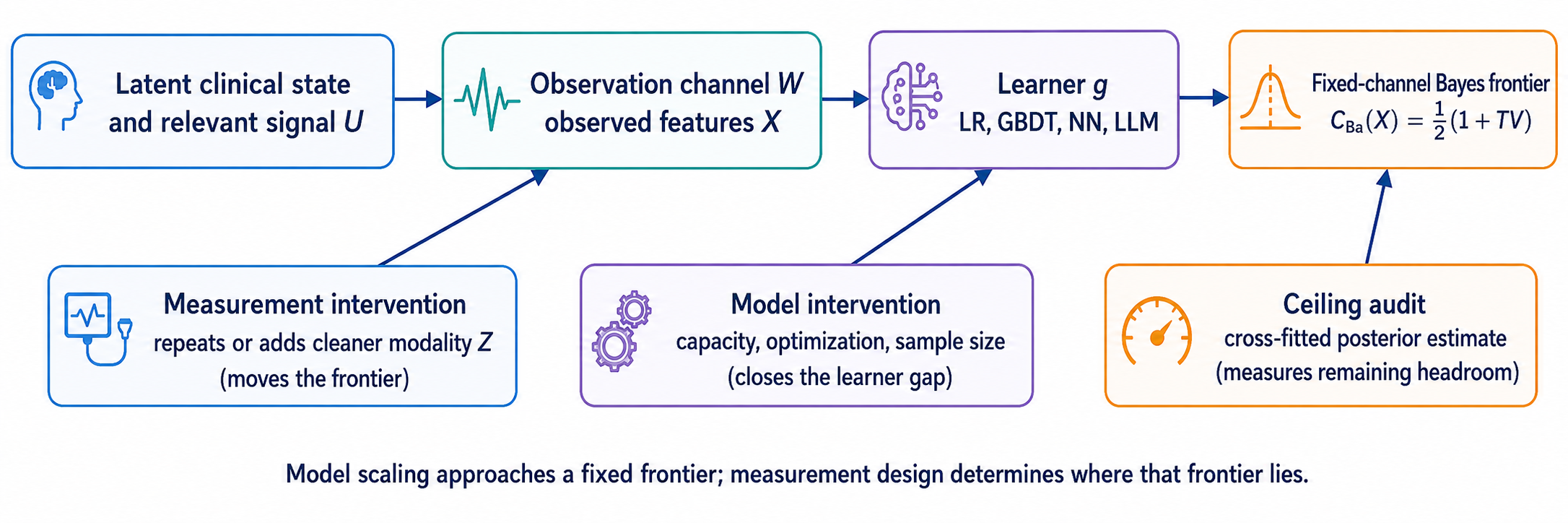}
\caption{Learner-gap and measurement-frontier decomposition. Scaling a model can approach $\cba(X)$; changing the observation channel can move it.}
\vspace{-7mm}
\label{fig:framework}
\end{figure}

The empirical study operates at two complementary scales. Three patient-level cohorts cover $99{,}343$, $253{,}680$, and $10{,}219$ observations, five channel configurations, grouped and ungrouped cross-fitting, administrative records, telephone-survey responses, questionnaire variables, and laboratory or examination measurements. Together, these cohorts provide a controlled evaluation of the proposed framework across heterogeneous measurement settings, allowing us to quantify learner headroom, assess estimator reliability, and isolate the contribution of complementary information channels. We then examine 104 task-level observations from more than 18 clinical categories to determine whether the same distinctions---diminishing same-channel gains and frontier shifts after measurement change---recur across the broader clinical literature. 
Our contributions are:
\begin{enumerate}
\item \textbf{A rigorous channel frontier.} We show that optimal balanced accuracy equals total-variation separation of class-conditional distributions; architecture invariance follows from data processing.
\item \textbf{Partial identification instead of post-hoc noise fitting.} Under shared replacement contamination, we give an exact ceiling and a sharp identified set: a ceiling of $0.85$ identifies separation $0.70$ but implies only $\alpha\in[0,0.30]$.
\item \textbf{An operational ceiling audit and criterion for added modalities.} A cross-fitted posterior estimator is paired with a consistency proof, a label-permutation optimism floor, and an underfit curve. Positive conditional mutual information is insufficient for hard-classification gain; strict improvement occurs exactly when Bayes decisions change on a positive-probability set, which we measure through a decision-flip rate.

\item \textbf{Multi-scale empirical validation.}
We evaluate the proposed audit on three real clinical cohorts containing
$99{,}343$, $253{,}680$, and $10{,}219$ observations across five measurement-channel
configurations. 
We complement them with a PRISMA-guided synthesis of
104 task-level observations spanning more than 18 clinical categories, showing
that the same distinction between diminishing same-channel gains and frontier
shifts after measurement change recurs across the broader clinical literature.
\end{enumerate}

\section{Related Work and Scope}
Bayes risk, binary hypothesis testing, and total variation provide the decision-theoretic foundation for prediction limits \cite{nielsen2014,jiao2019estimating}. Markov kernels contract divergence, and Dobrushin coefficients quantify total-variation contraction \cite{polyanskiy2015,gaubert2013}. Contamination models are classical in robust statistics \cite{huber1964}. Jiao, Han, and Weissman study estimation of fundamental limits, while Tao et al. estimate total variation discriminatively \cite{jiao2019estimating,tao2024}. Our target is narrower and operational: the equal-prior posterior functional corresponding to a clinical hard-decision rule. We pair it with diagnostics that determine whether a finite-sample estimate is trustworthy and use the resulting quantity to separate learner gap from channel ceiling on real cohorts. Note that we do not claim a new general estimator for every divergence-estimation setting.

The exact theory concerns balanced accuracy and Bayes $0$--$1$ risk. AUROC is a ranking functional, raw accuracy is prevalence dependent, and calibration is distinct from both. The literature synthesis therefore retains reported metrics and remains descriptive. The real-cohort audit reports AUROC and balanced accuracy separately and never treats them as interchangeable.

\section{Channel-Ceiling Theory}
Let $Y\in\{0,1\}$ have prevalence $\pi$, and let $P_y=\mathcal{L}(X\mid Y=y)$ have density $p_y$ with respect to a common measure. Define
\[
\begin{aligned}
\kappa_X&:=\TV(P_0,P_1)=\tfrac12\int|p_1-p_0|d\mu,\\
\BA(g)&:=\tfrac12\{\mathrm{TPR}(g)+\mathrm{TNR}(g)\}.
\end{aligned}
\]
We call $\kappa_X$ the \emph{effective channel separability}. In what follows, we present proof sketches for the theoretical results; complete proofs are provided in the supplementary material.

\begin{lemma}[Balanced-accuracy separation identity]
\label{lem:ba}
For any binary prediction problem,
\[
\cba(X):=\sup_g\BA(g)=\frac12(1+\kappa_X),
\]
attained by the equal-prior likelihood-ratio rule $\mathbf{1}\{p_1\ge p_0\}$. Under prevalence $\pi$,
\[
\begin{aligned}
\Rstar_\pi(X)
&=\int\min\{\pi p_1,(1-\pi)p_0\}d\mu\\
&=\tfrac12-\tfrac12\int|\pi p_1-(1-\pi)p_0|d\mu.
\end{aligned}
\]
\end{lemma}
\noindent\emph{Proof sketch.} For decision region $A$, $\BA=\tfrac12+\tfrac12\{P_1(A)-P_0(A)\}$; optimize over $A$. The raw-risk identity follows from $\min(a,b)=(a+b-|a-b|)/2$. 

\begin{lemma}[Data processing and architecture invariance]
\label{lem:dpi}
If $T$ is any deterministic or randomized representation of $X$, so $Y\to X\to T$, then
\[
\begin{aligned}
\TV\{\mathcal L(T\mid0),\mathcal L(T\mid1)\}&\le\kappa_X,\\
\cba(T)&\le\cba(X),\quad \Rstar_\pi(T)\ge\Rstar_\pi(X).
\end{aligned}
\]
\end{lemma}
\noindent\emph{Proof sketch.} Total variation contracts under Markov kernels. For raw risk, every rule based on $T$ is a restricted rule based on $X$. Lemma~\ref{lem:dpi} establishes a common upper bound; it does not assert that finite learners are equally close to it.

\subsection{Contaminated and General Measurement Channels}
Let $U$ denote informative latent content with class laws $Q_y$. A shared replacement channel returns class-independent content $R$ with probability $\alpha$:
\[
P_y=(1-\alpha)Q_y+\alpha R.
\]

\begin{theorem}[Exact replacement-contamination ceiling]
\label{thm:contam}
Let $\tau=\TV(Q_0,Q_1)$. Then
\[
\kappa_X=(1-\alpha)\tau,
\qquad
\cba(X)=\tfrac12\{1+(1-\alpha)\tau\}\le1-\alpha/2.
\]
Equality in the upper bound holds iff $\tau=1$. Moreover, $\Rstar_\pi(X)\ge\alpha\min\{\pi,1-\pi\}$.
\end{theorem}
\noindent\emph{Proof sketch.} The common $\alpha R$ component cancels in $P_1-P_0$; for raw risk, lower-bound both weighted densities by their common contamination component.

\begin{proposition}[Sharp identified set]
\label{prop:id}
If the model above holds and the population balanced-accuracy ceiling is $c\in[1/2,1]$, then without external knowledge of $\tau$ the sharp identified set is
\[
\mathcal I_\alpha(c)=[0,2(1-c)].
\]
Thus $c=0.85$ identifies $\kappa_X=0.70$ and only $\alpha\in[0,0.30]$.
\end{proposition}
\noindent\emph{Proof sketch.} The plateau identifies the product $(1-\alpha)\tau=2c-1$. Every $\alpha$ in the displayed interval is feasible with $\tau=(2c-1)/(1-\alpha)\le1$. Hence, $0.30$ is a limiting compatible value, not an estimated clinical noise rate.

\begin{theorem}[Dobrushin channel bound]
\label{thm:dobrushin}
For a common report kernel $W(dx\mid u)$ with Dobrushin coefficient
$\vartheta(W)=\sup_{u,u'}\TV\{W(\cdot\mid u),W(\cdot\mid u')\}$,
\[
\begin{aligned}
\kappa_X&\le\vartheta(W)\TV(Q_0,Q_1),\\
\cba(X)&\le\tfrac12\{1+\vartheta(W)\TV(Q_0,Q_1)\}.
\end{aligned}
\]
The shared replacement channel has $\vartheta(W)=1-\alpha$.
\end{theorem}
\noindent\emph{Proof sketch.} Apply the strong data-processing inequality for total variation. A class-dependent channel $W_y$ violates the common-channel assumption and may create as well as destroy apparent separation.

\subsection{Estimating the Fixed-Channel Frontier}
Let $M=\tfrac12(P_0+P_1)$ and $\eta_{\rm eq}(x)=p_1(x)/(p_0(x)+p_1(x))$.

\begin{proposition}[Posterior representation]
\label{prop:post}
\[
\begin{aligned}
\kappa_X&=\mathbb E_{X\sim M}|2\eta_{\rm eq}(X)-1|,\\
\cba(X)&=\tfrac12\{1+\mathbb E_M|2\eta_{\rm eq}(X)-1|\}.
\end{aligned}
\]
\end{proposition}
\noindent\emph{Proof sketch.} Substitute the mixture density $(p_0+p_1)/2$ into the expectation; the denominator cancels and yields $\tfrac12\int|p_1-p_0|$.

On a balanced sample, partition observations into $K$ folds, fit a probabilistic learner on the remaining folds, and collect out-of-fold predictions $\widehat\eta_{-k(i)}(X_i)$. Define
\[
\widehat\kappa_{\rm CF}=\frac1n\sum_{i=1}^n|2\widehat\eta_{-k(i)}(X_i)-1|,
\qquad
\widehat\cba_{\rm CF}=\tfrac12(1+\widehat\kappa_{\rm CF}).
\]

\begin{proposition}[Consistency of the cross-fitted audit]
\label{prop:consistency}
If each out-of-fold posterior estimator is $L^1(M)$-consistent and fold sizes diverge, then
$\widehat\kappa_{\rm CF}\to_p\kappa_X$ and $\widehat\cba_{\rm CF}\to_p\cba(X)$.
\end{proposition}
\noindent\emph{Proof sketch.} The map $a\mapsto|2a-1|$ is $2$-Lipschitz. Posterior $L^1$ error controls plug-in error, and cross-fitting permits a foldwise law of large numbers. In practice, underfitting tends to bias the estimate downward; multiple flexible learners and bootstrap intervals should be reported.

\subsection{Prospective Measurement Interventions}
A meaningful reliability claim requires an explicit measurement model rather than equating contamination probability with Cronbach's alpha. Suppose
$U\mid Y=y\sim\mathcal N(\mu_y,\sigma_U^2)$ and repeated reports satisfy
$X_j=U+\varepsilon_j$, with independent $\varepsilon_j\sim\mathcal N(0,\sigma_\varepsilon^2)$. Let
$\rho=\sigma_U^2/(\sigma_U^2+\sigma_\varepsilon^2)$ and $d'_U=|\mu_1-\mu_0|/\sigma_U$.

\begin{theorem}[Reliability--repetition ceiling law]
\label{thm:repeat}
For the average $\bar X_m=m^{-1}\sum_{j=1}^mX_j$,
\[
\begin{aligned}
d'_m&=d'_U\sqrt{\frac{m\rho}{1+(m-1)\rho}},\\
\cba(\bar X_m)&=\Phi\!\left(
\frac{d'_U}{2}\sqrt{\frac{m\rho}{1+(m-1)\rho}}\right).
\end{aligned}
\]
For $d'_U>0$, the ceiling increases with $m$ when $\rho<1$ and converges to the latent-score ceiling $\Phi(d'_U/2)$.
\end{theorem}
\noindent\emph{Proof sketch.} Averaging reduces error variance to $\sigma_\varepsilon^2/m$; equal-variance Gaussian discrimination has balanced accuracy $\Phi(d'/2)$. The result gives a forward, independently parameterized prediction.

\begin{theorem}[Multimodal non-decrease and strictness]
\label{thm:multi}
For an auxiliary modality $Z$, $\cba(X,Z)\ge\cba(X)$, with strict inequality iff
\[
\TV(P^{XZ}_0,P^{XZ}_1)>\TV(P^X_0,P^X_1).
\]
For raw $0$--$1$ risk, with $\eta(X,Z)=P(Y=1\mid X,Z)$ and $\eta_X(X)=P(Y=1\mid X)$,
\[
\begin{aligned}
\Rstar(X)-\Rstar(X,Z)
&=\mathbb E|\eta(X,Z)-1/2|\\
&\quad-\mathbb E|\eta_X(X)-1/2|\ge0.
\end{aligned}
\]
Equality holds iff, conditional on almost every $X=x$, $\eta(x,Z)-1/2$ does not change sign almost surely.
\end{theorem}
\noindent\emph{Proof sketch.} Marginalization from $(X,Z)$ to $X$ contracts total variation. The raw-risk result follows from conditional Jensen. Positive $I(Z;Y\mid X)$ alone may refine confidence without crossing a decision boundary.

\begin{proposition}[Gaussian multimodal complementarity]
\label{prop:gaussmulti}
Let $X^{(1)},\ldots,X^{(M)}$ be conditionally independent given $Y$, with
$X^{(j)}\mid Y=y\sim\mathcal N(\mu_{jy},\Sigma_j)$ and common within-modality covariance. Define
\[
d_j^2=(\mu_{j1}-\mu_{j0})^\top
\Sigma_j^{-1}(\mu_{j1}-\mu_{j0}).
\]
Then
\[
\cba(X^{(1)},\ldots,X^{(M)})
=\Phi\!\left(\frac12\sqrt{\sum_{j=1}^Md_j^2}\right).
\]
Any added modality with $d_j>0$ strictly raises a non-perfect joint frontier.
\end{proposition}
\noindent\emph{Proof sketch.} Conditional independence gives block-diagonal covariance, so squared Mahalanobis separations add. Importantly, $\max\{\cba(X),\cba(Z)\}$ is only a lower bound on the joint ceiling.

\section{Controlled Validation}
We retain simulations only where the population frontier is analytically known. Figure~\ref{fig:synthetic} summarizes four complementary checks. In the replacement experiment (a), latent classes have disjoint nonlinear supports and the full feature vector is replaced by a class-independent draw with probability $\alpha$, so Theorem~\ref{thm:contam} gives the exact envelope $\cba=1-\alpha/2$. Flexible nonlinear learners approach this frontier, whereas logistic regression remains below it because its decision class cannot express the radial boundary; at $\alpha=0.30$, the controlled ceiling is exactly $0.85$. In the Gaussian multimodal experiment (b), the marginal ceilings are $0.80$ for $X$ and $0.90$ for $Z$, while Proposition~\ref{prop:gaussmulti} gives the complementary joint ceiling $0.9374$, which exceeds both marginals rather than merely matching their maximum. Panel (c) illustrates Theorem~\ref{thm:repeat}: repeated measurements raise the frontier along the predicted saturating curves for $\rho\in\{0.30,0.60,0.90\}$ but cannot exceed the latent-score ceiling $0.95$. This remains a prospective theoretical prediction and is not empirically validated by the three real cohorts. Finally, panel (d) applies the cross-fitted estimator to the contamination experiment and closely recovers the known frontier across $\alpha\in\{0,.1,\ldots,.5\}$. Together, the panels verify the exact population identities, the distinction between marginal and complementary channels, and the ability of the proposed audit to recover a known fixed-channel frontier under controlled conditions.

\begin{figure}[t]
\centering
\begin{minipage}[t]{.49\columnwidth}
\centering
\includegraphics[width=\linewidth]{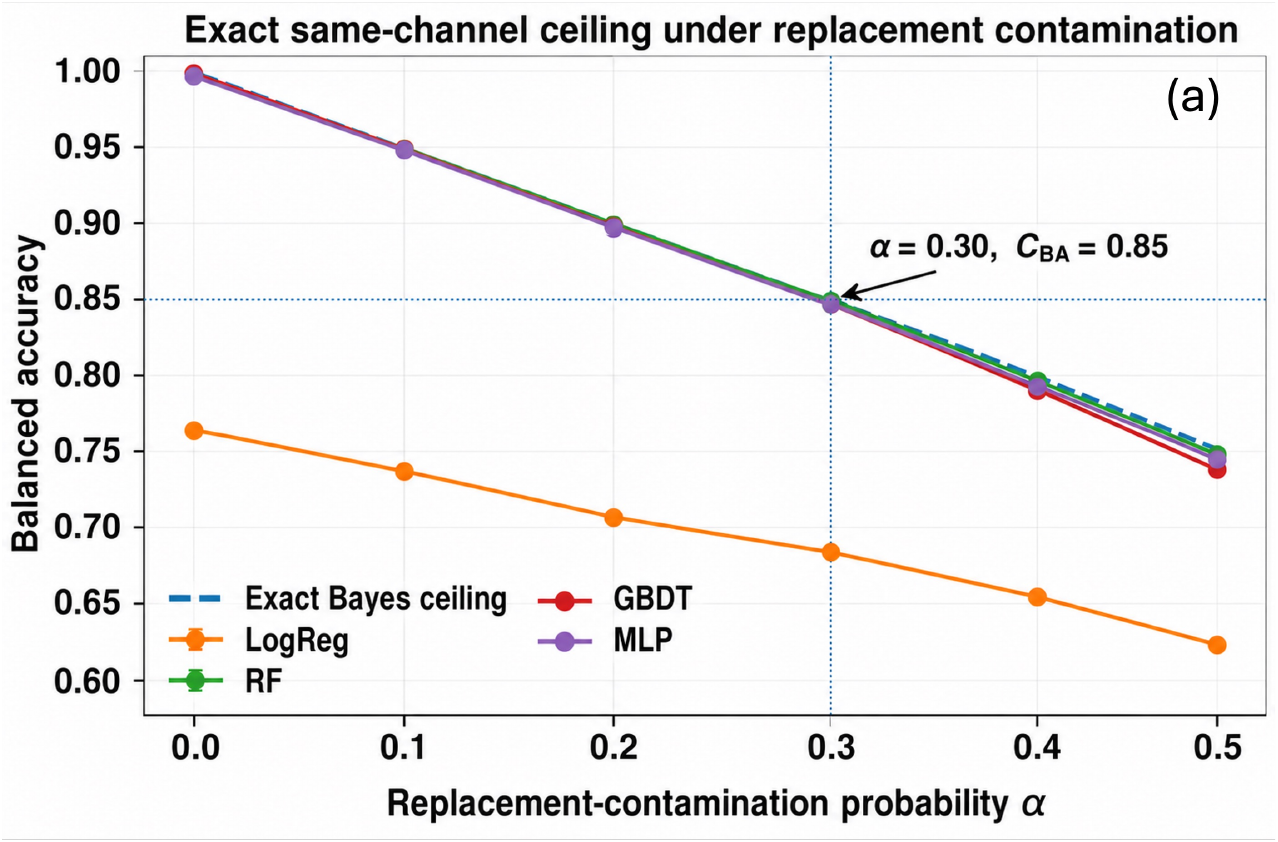}
\end{minipage}\hfill
\begin{minipage}[t]{.49\columnwidth}
\centering
\includegraphics[width=\linewidth]{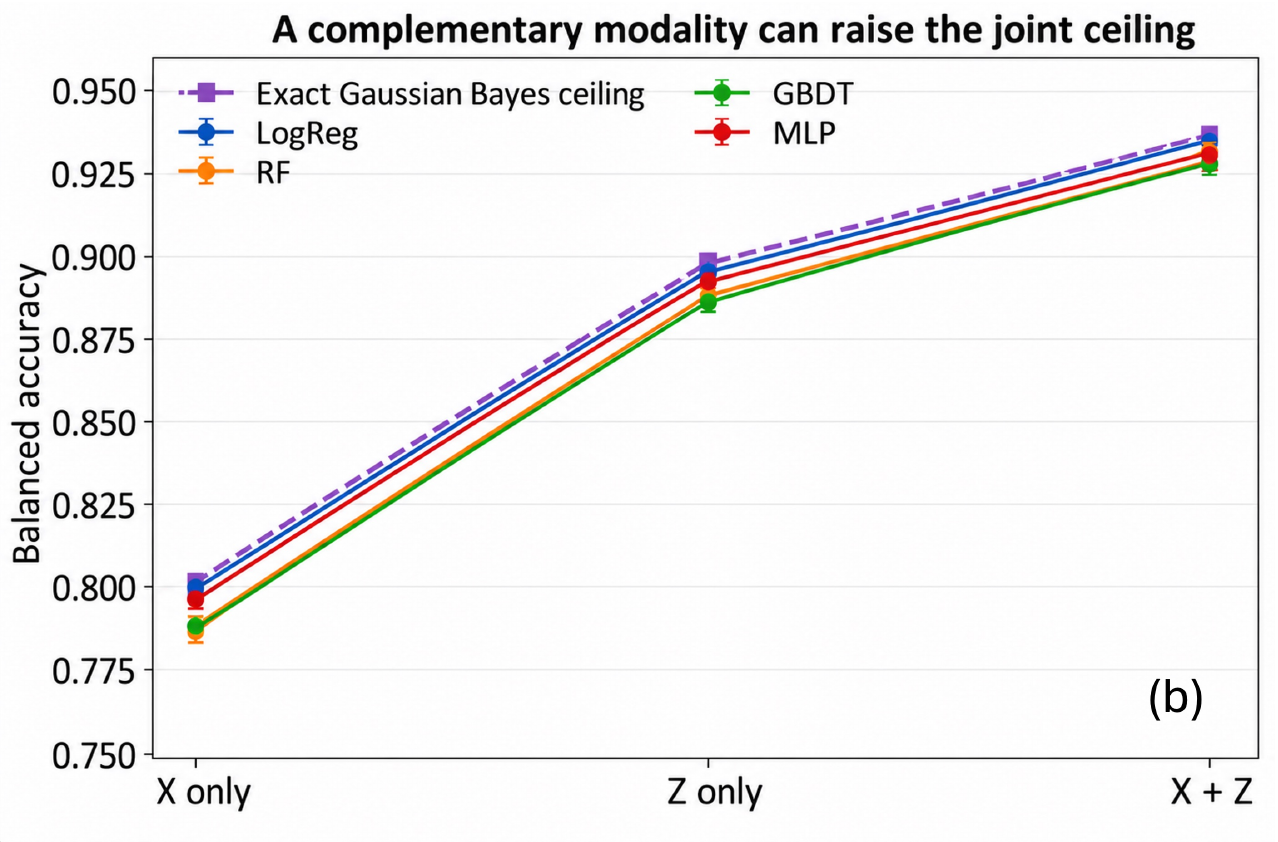}
\end{minipage}

\begin{minipage}[t]{.49\columnwidth}
\centering
\includegraphics[width=\linewidth]{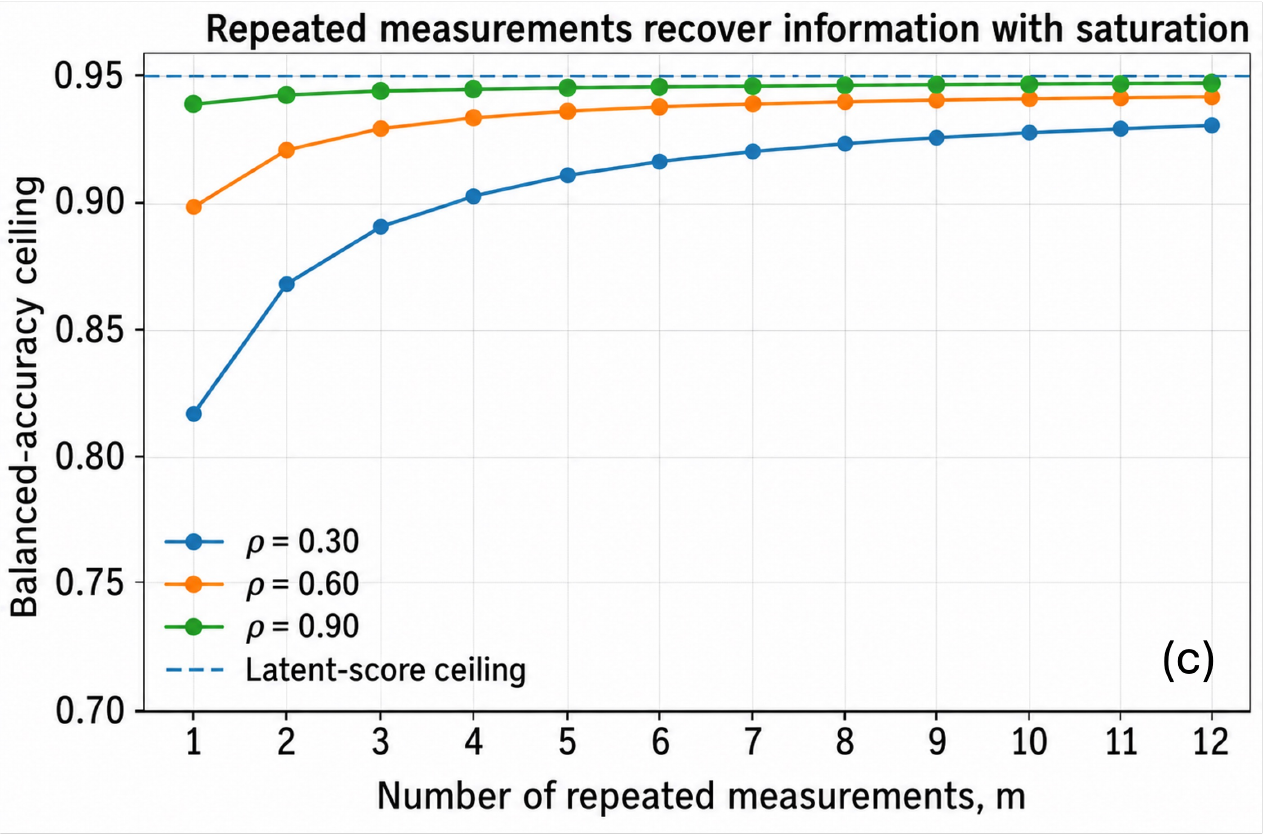}
\end{minipage}\hfill
\begin{minipage}[t]{.49\columnwidth}
\centering
\includegraphics[width=\linewidth]{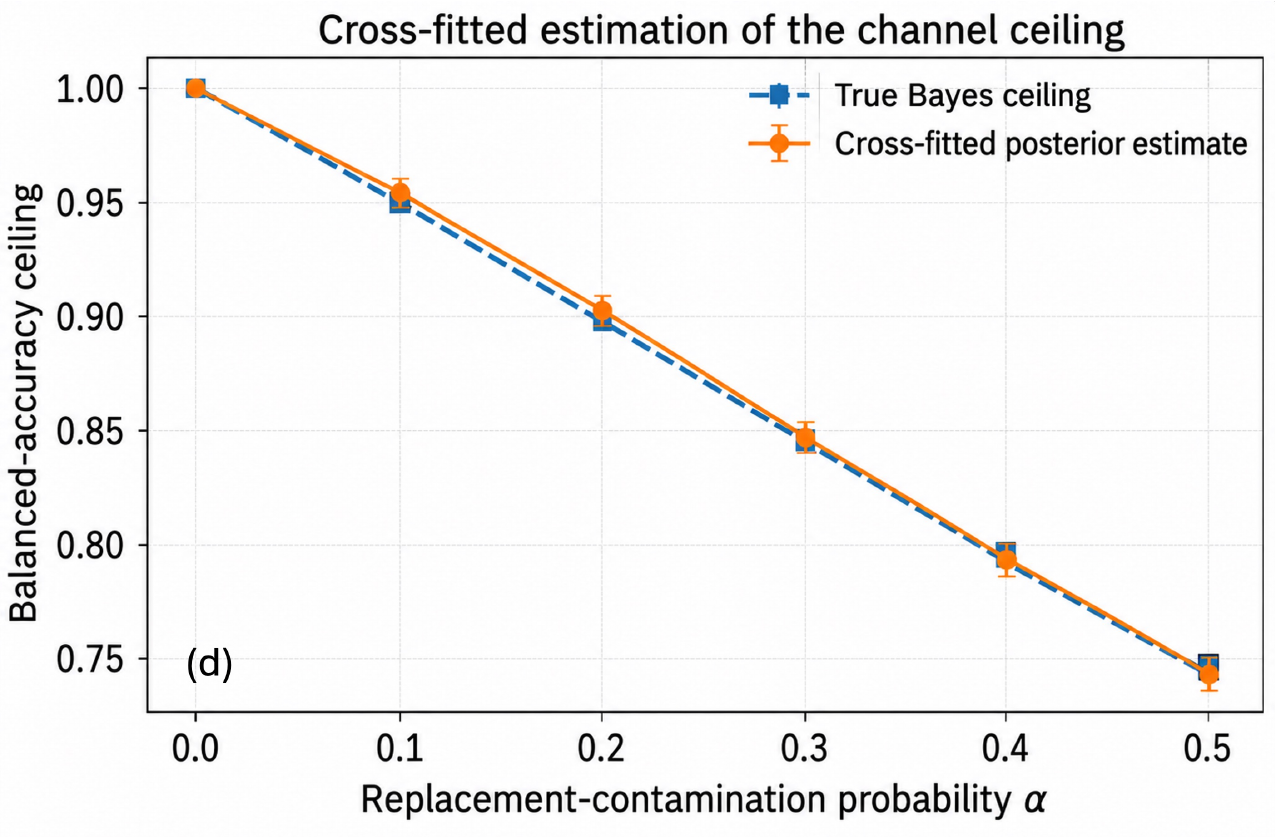}
\end{minipage}
\caption{Controlled validation with known frontiers. (a) Replacement-contamination envelope. (b) Gaussian complementarity, with the joint frontier exceeding both marginals. (c) Reliability--repetition curves approaching the latent-score ceiling. (d) Cross-fitted recovery of the contamination frontier. Panels (a), (b), and (d) validate the audit mechanisms; panel (c) is a prospective prediction.}
\vspace{-7mm}
\label{fig:synthetic}
\end{figure}

\section{Real-Cohort Frontier Audits}
\subsection{Audit Design and Cohorts}
For each cohort, probabilistic learners generate out-of-fold posterior estimates. We evaluate the equal-prior functional in Proposition~\ref{prop:post}, report its bootstrap interval, and define learner headroom as $G=\widehat\cba_{\rm CF}-\BA_{\rm best}$. The permutation-null optimism floor is $\widehat\cba_{\rm perm}-0.5$. The underfit curve refits the posterior learner at training fractions $0.25/0.5/0.75/1.0$. A positive final increment indicates that the ceiling estimate remains a lower bound; stabilization or small oscillation supports convergence. Full algorithms, importance weighting, preprocessing, and hyperparameters are in the supplement. Implementations use scikit-learn \cite{pedregosa2011scikit}.

The UCI Diabetes 130-US Hospitals cohort \cite{strack2014impact} contains $99{,}343$ encounters from $69{,}990$ patients after removing death and hospice discharges from $101{,}766$ raw encounters; prevalence is $0.1139$. Cross-fitting and bootstrap resampling are patient-grouped. BRFSS 2015 contributes $253{,}680$ respondents with prevalence $0.1393$; all predictors are telephone-survey self-reports \cite{cdcbrfss2015}. NHANES 2015--2018 contains $10{,}219$ adults with measured glycohemoglobin and prevalence $0.1409$ \cite{nhanes20152016,nhanes20172018}. Its outcome is HbA1c $\ge6.5\%$; questionnaire, measured, and joint channels are audited after explicitly excluding all glycemic analytes from the predictors.

\begin{table*}[t]
\centering
\scriptsize
\caption{Three-cohort frontier audit. All ceilings and achieved values are balanced accuracy; AUROC is reported in its own column. ``Floor'' is the permutation-null excess above $0.5$.}
\label{tab:audit}
\begin{tabular}{lrrrrrrl}
\toprule
Cohort & $n$ & Prev. & Ceiling (95\% CI) & Floor & AUROC & Best BA / $G$ & Underfit verdict\\
\midrule
UCI readmission & $99{,}343$ & $0.1139$ & $0.6225$ $[0.6213,0.6240]$ & $+0.0147$ & $0.6692$ & $0.6223/+0.0002$ & lower bound; final $+0.0029$\\
BRFSS diabetes & $253{,}680$ & $0.1393$ & $0.7522$ $[0.7515,0.7529]$ & $+0.0045$ & $0.8298$ & $0.7518/+0.0003$ & converged; final $-0.0002$\\
NHANES HbA1c, joint & $10{,}219$ & $0.1409$ & $0.7623$ $[0.7582,0.7659]$ & $+0.0309$ & $0.8419$ & $0.7541/+0.0081$ & converged; oscillating\\
\bottomrule
\end{tabular}
\end{table*}

\paragraph{Metric scale.}
The literature band is AUROC, whereas the audit frontier is balanced accuracy. Under a single-index approximation,
$\BA=\Phi((\sqrt{2}/2)\Phi^{-1}(\mathrm{AUROC}))$: AUROC $0.78/0.80/0.85/0.88$ maps to BA $0.707/0.724/0.768/0.797$. Thus the literature region corresponds approximately to BA $0.71$--$0.80$, which is the converted band shaded in Figure~\ref{fig:frontiers}. We do not compare a BA ceiling to the unconverted AUROC interval.

\begin{figure}[t]
\centering
\includegraphics[width=.98\columnwidth]{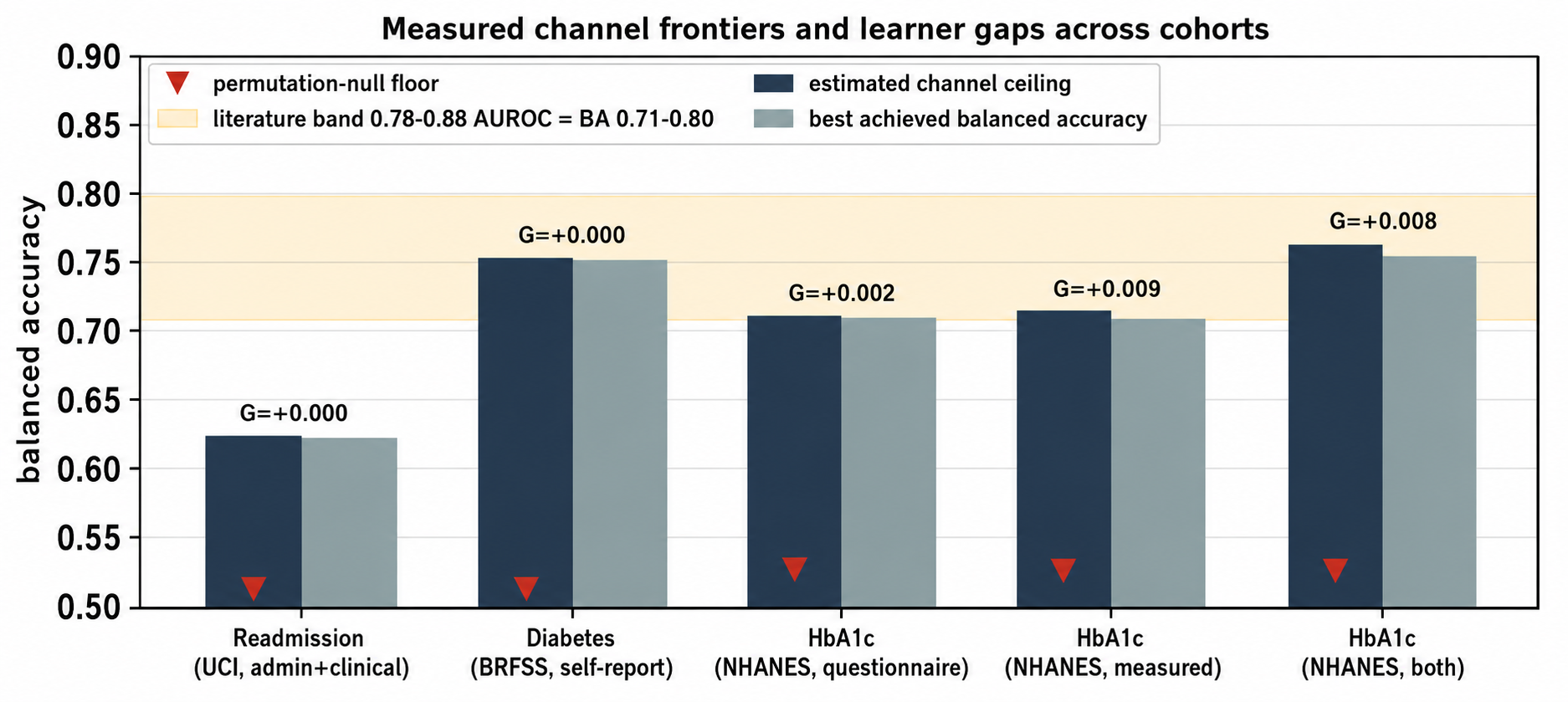}
\caption{Measured channel frontiers and learner gaps across five configurations. The shaded region is the literature reported balanced accuracy band. UCI lies below it; BRFSS and the NHANES joint channel lie within it. These placements describe the cohorts and do not validate a universal band.}
\vspace{-7mm}
\label{fig:frontiers}
\end{figure}

\subsection{Cohort Results}
\paragraph{UCI readmission: the low result is channel-limited.}
The ceiling is $0.6225$ with 95\% CI $[0.6213,0.6240]$, and the best balanced accuracy is $0.6223$, giving $G=+0.0002$. The published study on this cohort reports XGBoost AUROC $0.667$ \cite{emijohnson2025predicting}; our patient-grouped cross-fitted GBDT obtains $0.6692$, a difference of $+0.0022$ under stricter validation than its encounter-level 80/20 split. The audit therefore measures the same task reported in the literature. It also resolves the earlier speculative appeal to unobserved social determinants: the observed $0.667$ is already close to what the recorded variables support, rather than evidence of a failed algorithm. Because the underfit curve ends $0.6102,0.6116,0.6196,0.6225$, with final increment $+0.0029$, the estimated frontier is reported honestly as a lower bound. Splitting administrative channel $A$ from clinical channel $B$ yields ceilings $0.5753$, $0.5961$, and $0.6206$ for $A$, $B$, and $A+B$; AUROC rises $0.6058\to0.6674$, while the decision-flip rate is $0.2996$ and risk gain is $0.0454$.

\paragraph{BRFSS: the strongest audit.}
BRFSS has ceiling $0.7522$ $[0.7515,0.7529]$, best balanced accuracy $0.7518$, $G=+0.0003$, and the smallest optimism floor, $+0.0045$, consistent with its largest sample. Its underfit sequence $0.7478,0.7520,0.7524,0.7522$ has final increment $-0.0002$ and is treated as converged. The perception channel $A$ (GenHlth, MentHlth, PhysHlth, DiffWalk) and recalled-diagnosis channel $B$ (HighBP, HighChol, CholCheck, Stroke, HeartDiseaseorAttack, BMI) have ceilings $0.6917$ and $0.7165$; their joint ceiling is $0.7413$, a complementarity gain of $+0.0248$. Recalled diagnoses transmit prior objective measurement through memory, a different pathway from subjective symptom perception and a concrete instance of recall-mediated channel distortion. AUROC rises $0.7488\to0.8160$, flip rate is $0.1907$, and risk gain is $0.0496$.

\paragraph{NHANES: a null marginal contrast and significant complementarity.}
The questionnaire channel has ceiling $0.7110$ $[0.7082,0.7141]$; the measured channel has $0.7152$ $[0.7118,0.7190]$. Their difference is $+0.0042$, and the intervals overlap: this is a null result, not evidence that measured variables dominate. The joint frontier is $0.7623$ $[0.7582,0.7659]$, giving complementarity $+0.0471$; its interval is disjoint from the measured-channel interval. The gain therefore arises from complementary decision information, not from either marginal channel being intrinsically cleaner. The joint AUROC is $0.8419$, the flip rate is $0.2105$, and risk gain is $0.0513$. The underfit curve $0.7456,0.7639,0.7602,0.7623$ drops $0.0037$ between fractions $0.5$ and $0.75$ and is treated as oscillating within noise rather than rising.

We used HistGB \cite{ke2017lightgbm} with early stopping, \texttt{max\_leaf\_nodes=15}, \texttt{min\_samples\_leaf=50}, $\ell_2=1.0$, learning rate $0.05$, native NaN handling, and no sentinel imputation; the floor falls to $+0.0309$. Our diagnostic results demonstrate 
that the plug-in estimator requires regularization or larger samples for performing its intended function.

\begin{figure}[t]
\centering
\includegraphics[width=.99\columnwidth]{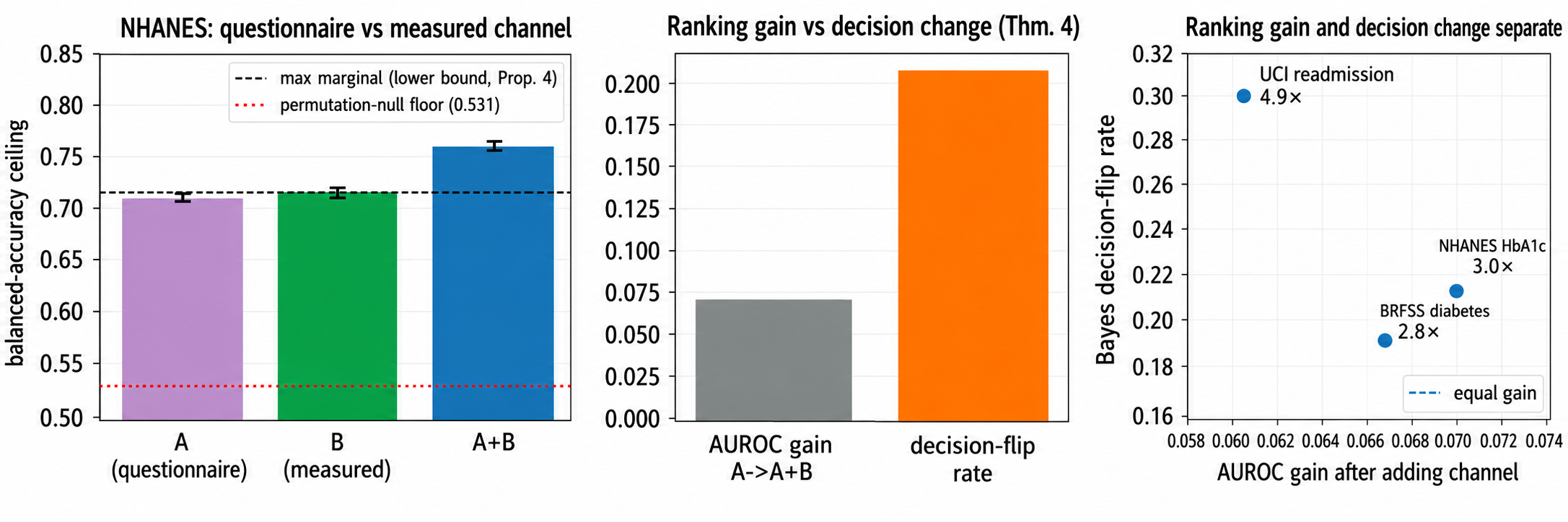}
\caption{Channel complementarity and ranking--decision separation. Left: NHANES marginal frontiers overlap, while the joint frontier is higher. Middle: the joint channel yields AUROC gain $+0.0705$ with decision-flip rate $0.2105$. Right: this separation replicates across UCI, BRFSS, and NHANES, with AUROC gains $+0.0617/+0.0672/+0.0705$ and flip rates $0.2996/0.1907/0.2105$.}
\vspace{-7mm}
\label{fig:complementarity}
\end{figure}

\subsection{Two Replicated Regularities}
The cohort experiments reveal two replicated regularities.
Figure~\ref{fig:complementarity} summarizes the separation between
ranking improvement and Bayes decision change, while
Table~\ref{tab:learners} compares performance and 
frontiers across learners.

\textbf{Ranking gain and decision change separate.} Across channels, AUROC gains are $+0.0617$, $+0.0672$, and $+0.0705$, whereas flip rates are $0.2996$, $0.1907$, and $0.2105$: decision changes are $4.9\times$, $2.8\times$, and $3.0\times$ larger. The same dissociation appears across learners. The MLP reaches AUROC $0.8180$ on BRFSS and $0.8385$ on NHANES while balanced accuracy collapses to $0.5793$ and $0.5724$, leaving gaps $+0.1642$ and $+0.1933$. In UCI, RF ranks second by AUROC at $0.6648$ but has balanced accuracy $0.5349$ and gap $+0.0861$. Good ranking can coexist with poor thresholded decisions, precisely the distinction formalized by Theorem~\ref{thm:multi}.

\textbf{Architecture invariance is visible within cohorts.}
Table~\ref{tab:learners} shows that, across logistic regression, random
forest, gradient boosting, and MLP, the estimated frontiers span only
$0.0087$ on BRFSS and $0.0195$ on NHANES, whereas achieved balanced
accuracy spans $0.173$ and $0.182$, respectively. Thus, learners with
substantially different decision performance nevertheless recover similar
fixed-channel frontiers. This within-cohort comparison removes the disease,
dataset, and metric confounding present in the literature-level model-family
analysis. It also addresses a potential tautology concern: if $G\approx0$
were mechanically induced because the frontier and achieved decisions are
derived from the same fitted posterior, the audit could not reveal gaps of
$+0.0861$ for RF on UCI or $+0.1642$ and $+0.1933$ for MLP on BRFSS and
NHANES. The near-zero GBDT gaps are therefore empirical findings rather
than algebraic artifacts.

\begin{table*}[t]
\centering
\scriptsize
\caption{Each entry is achieved BA/estimated ceiling. Large deficient-learner gaps coexist with tightly clustered frontier estimates.}
\label{tab:learners}
\begin{tabular}{lccc}
\toprule
Learner & UCI readmission & BRFSS diabetes & NHANES joint HbA1c\\
\midrule
Logistic regression & $0.6016/0.6055$ & $0.7461/0.7481$ & $0.7426/0.7461$\\
Random forest & $0.5349/0.6211$ & $0.7269/0.7478$ & $0.7358/0.7554$\\
Gradient boosting & $0.6223/0.6225$ & $0.7518/0.7522$ & $0.7541/0.7623$\\
MLP & $0.5146/0.5673$ & $0.5793/0.7435$ & $0.5724/0.7656$\\
\midrule
Ceiling spread & --- & $0.0087$ & $0.0195$\\
Achieved-BA spread & --- & $0.173$ & $0.182$\\
\bottomrule
\end{tabular}
\end{table*}

\section{Large-Scale Empirical Observations Across Clinical Prediction}
The cohort audits test the method under fixed outcomes, channels, and validation designs. To assess external scope, we conduct a PRISMA-guided umbrella synthesis of 
104 task-level observations spanning more than 18 disease categories \cite{page2021prisma,aromataris2015summarizing}, (see \textit{supplementary section 7.1, Table 5}). Outcomes, horizons, prevalence, validation, and metrics differ, so the results are not pooled as a common estimand, rather we test whether the qualitative learner--channel patterns recur across clinical domains.

\begin{figure*}[t]
\centering
\includegraphics[width=.9\textwidth]{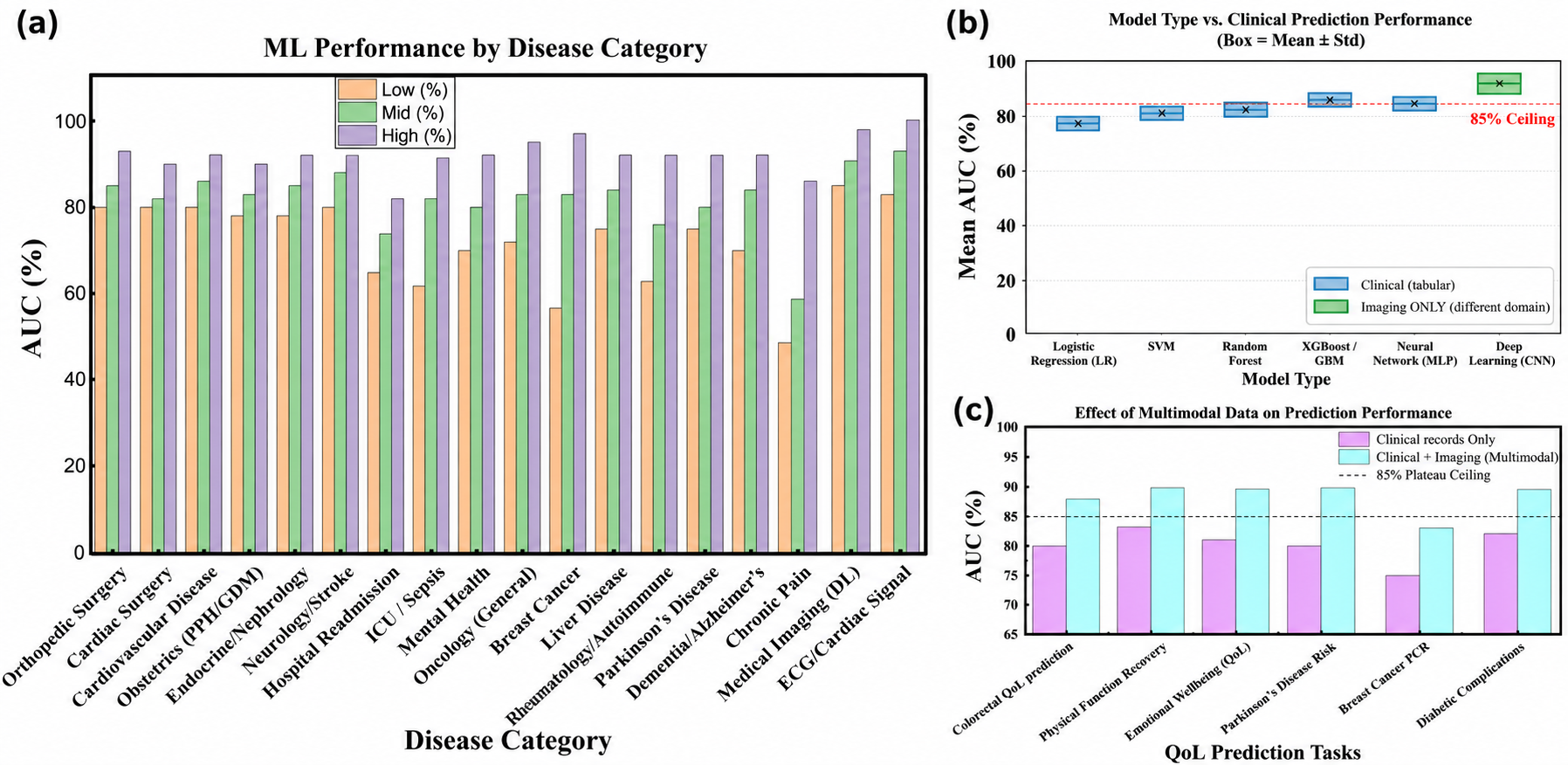}
\caption{Large-scale clinical observations. (a) Disease-category summaries show a broad, non-universal performance region, with readmission and chronic pain lower and ECG/imaging higher. (b) Most model-family gains occur before strong nonlinear tabular learners; imaging CNNs operate on a different channel. (c) Six comparisons show higher performance after multimodal channel expansion. Values are descriptive, heterogeneous summaries rather than pooled frontier estimates.}
\vspace{-7mm}
\label{fig:literature}
\end{figure*}

\paragraph{Cross-domain recurrence without a universal constant.}
Figure~\ref{fig:literature}(a) shows repeated intersection with an AUROC region near $0.78$--$0.88$ across surgical, cardiovascular, obstetric, endocrine/renal, neurological, and oncological tasks \cite{karimi2024accuracy,ogink2021wide,sinha2023comparison,liu2025machine,ranjbar2023predicting,yang2023predictive}. The broad ranges and counterexamples are equally important: chronic-pain PROM prediction and readmission extend lower, whereas ECG and imaging extend higher \cite{zmudzki2023machine,tseng2021prediction,xie2024machine}. This is consistent with task-specific frontiers that can occupy a similar region without sharing one universal ceiling. The cohort audits anchor that interpretation: BRFSS and joint-channel NHANES fall inside the converted contextual band, whereas UCI readmission lies below it because its recorded channel supports a lower frontier.
\vspace{-3mm}
\paragraph{Learner saturation and channel expansion.}
Figure~\ref{fig:literature} (b) rises from logistic regression through boosting, but the increment from boosting to deeper tabular models is small relative to the earlier gain. This cross-study pattern is consistent with a learner gap that narrows as models better exploit a fixed channel. In this context, we also perform within-cohort test and find that BRFSS and NHANES frontier estimates span only $0.0087$ and $0.0195$ across four architectures while achieved balanced accuracy spans $0.173$ and $0.182$. In contrast, Figure~\ref{fig:literature}(c) reports clinical-only values near $0.75$--$0.83$ and multimodal values near $0.83$--$0.90$, including genomic--clinical Parkinson prediction with AUC $0.897$ \cite{makarious2022multimodality,bektas2022machine,dammu2023deep,olivera2023comparison}. These observations align with Theorem~\ref{thm:multi}, while NHANES supplies the necessary refinement in Figure~\ref{fig:complementarity} (a): neither marginal channel is superior, yet their joint frontier increases. Complementarity and decision change, not the label ``objective,'' determine the gain.

\textbf{A fixed-outcome dementia contrast further separates channel richness from sample size.}
The review by \citet{veronese2025clinical} is largely multimodal or biomarker-rich and reports widely dispersed AUROCs, so its mean of $0.845$ should not be treated as a large-sample structured-record cluster. Within that review, a claims-only study of $117{,}895$ individuals achieved modest discrimination, whereas smaller memory-clinic and neuroimaging cohorts often exceeded $0.85$--$0.90$ \cite{reinke2023dementia}, suggesting that performance may follow channel richness more closely than sample size alone. Further details and additional results of our overall synthesis is given in \textit{supplementary sections 6 and 7}. 


\section{Discussion and Limitations}
\vspace{-1mm}
The evidence forms three connected layers. Controlled simulations verify the mathematical mechanisms when the population frontier is known; patient-level cohort audits estimate task-specific frontiers and learner gaps; and the 104-task synthesis shows that the same qualitative distinctions recur across diseases and measurement regimes. The real cohorts validate the \emph{audit method}, not a universal AUROC band. UCI resolves a known low-performing readmission task: $G=+0.0002$ indicates that the published $0.667$ AUROC largely reflects the recorded variables rather than an algorithmic failure, although the rising underfit curve makes the frontier estimate a lower bound. BRFSS supplies the cleanest converged audit, and NHANES supplies the strongest theoretical refinement: statistically indistinguishable marginal channels can still produce a significantly higher joint frontier. Across both the cohort and literature scales, conditional information and decision change matter more than modality labels alone.

The permutation diagnostic is not optional. The discarded NHANES run demonstrates that a highly flexible posterior can inflate $|2\widehat\eta-1|$ on noise, particularly with small samples and artificial sentinel partitions. Conversely, an underfit posterior can suppress the functional and make a ceiling estimate only a lower bound. Bootstrap intervals quantify sampling variability but do not remove either bias. External validation, site shift, and temporal drift remain separate concerns because the population frontier itself can change across deployment environments.

Theorem~\ref{thm:repeat} is not empirically validated here. It remains a prospective prediction for repeated independent measurements under an additive equal-variance Gaussian model. The contamination identity assumes a shared class-independent replacement component; the Gaussian multimodal law assumes conditional independence and equal within-class covariances. The distribution-free non-decrease and strictness statements remain valid more broadly, but empirical equality or strictness is subject to finite-sample estimation noise -- adding uninformative coordinates can make a joint estimate slightly lower than a marginal estimate even though the population functional cannot decrease.

The large-scale synthesis remains descriptive, reviews overlap, metrics and validation designs differ, and patient-level uncertainty is often unavailable. Its PRISMA counts use separate units for screened records, included source publications, and extracted task observations, which the supplement reports explicitly. These limitations prevent a pooled frontier estimate, but they do not erase the repeated qualitative contrasts in Figure~\ref{fig:literature}. The real cohorts are also observational and do not prove that changing a measurement will causally improve outcomes. The actionable conclusion is narrower: report both achieved performance and an audited fixed-channel frontier. A large learner gap motivates model improvement, whereas, a smaller gap should shift attention toward measurement, labels, and deployment context.

\section{Implications for Clinical-AI Study Design}
\paragraph{Report a frontier audit, not only a leaderboard.}
A benchmark should report the best achieved balanced accuracy, the cross-fitted frontier, their gap $G$, the permutation-null floor, and the underfit verdict. The achieved value describes the fitted learner; $G$ measures extractive headroom; the null floor measures finite-sample optimism; and the underfit curve determines whether the frontier is stable or only a lower bound. Near-zero headroom is conditional on the audited variables and validation distribution, not a declaration that the outcome is intrinsically unpredictable. The UCI and BRFSS panels demonstrate the reason -- boosting can have negligible headroom while RF or MLP retains a large gap on the same cohort.

\paragraph{Separate ranking, decisions, and calibration.}
An added channel may improve ordering modestly while moving many posterior probabilities across the decision boundary; conversely, a learner may preserve AUROC while producing poor thresholded decisions. Clinical studies should therefore report a ranking metric, a prevalence-robust decision metric, calibration, and a prespecified threshold-selection protocol. Theorem~\ref{thm:multi} characterizes population hard-decision gain, while the large MLP gaps show that practical calibration and thresholding failures remain architecture dependent.

\paragraph{Ablate channels at fixed cohorts and outcomes.}
Claims that imaging, laboratory variables, or questionnaires move a frontier are strongest when $X$, $Z$, and $(X,Z)$ are evaluated on the same patients, outcome, split, metric, and comparable learner families. NHANES illustrates the payoff: neither marginal channel dominates, but the joint channel is complementary. More data within a fixed channel can reduce variance and close learner gaps, as the stable BRFSS audit suggests, but sample size alone does not change $P_0$ and $P_1$. Cohort-specific underfit curves are therefore more informative than cross-study sample-size plots.

\paragraph{Treat measurement change as the next experiment when headroom is small.}
A small $G$ changes the intervention rather than ending the task. Candidate actions include repeated administration, more granular temporal features, adjudicated outcomes, or a complementary modality selected for conditional information. The reliability--repetition theorem provides one prospective design but is not empirically validated here. More generally, a frontier shift should be demonstrated by repeating the audit after measurement change, with uncertainty on the frontier difference and a decision-flip analysis.

\section{Conclusion}
Clinical prediction has two scaling problems. The learner determines how closely a fitted model approaches the information already recorded; the measurement channel determines the population frontier. Total-variation theory makes the distinction exact, and cross-fitted audits with permutation and underfit diagnostics make it measurable. Across three cohorts, near-zero gaps for well-tuned boosting coexist with large gaps for deficient learners, while channel complementarity changes decisions far more often than AUROC gains alone suggest. Across 104 additional clinical tasks, the same qualitative pattern recurs: same-channel model gains diminish, whereas richer or complementary measurement channels often extend performance. Together, the cohort audits and large-scale observations indicate whether the next investment should be a larger model or a better measurement.

\section*{Supplementary Material}

\section{Background}
This supplement contains four components: (i) complete proofs for every lemma, theorem, and proposition stated in the main paper; (ii) details of cross-fitted frontier audit and real-cohort experimental details; (iii) the complete PRISMA-guided evidence-synthesis protocol and descriptive tables; and (iv) additional empirical figures supporting the source meta-analysis in addition to those in the the main paper, each accompanied by detailed interpretation and methodological qualification. 

\section{A Channel-Ceiling Theory of Clinical Prediction}
\label{sec:theory}
\subsection{Motivation and contribution}

Clinical prediction performance is determined jointly by the information present in the measurement channel and by how effectively a learning algorithm extracts that information. The first component is a property of the observed data distribution; the second is a property of finite-sample estimation, optimization, and model class. We formalize this distinction using balanced accuracy because raw accuracy varies with prevalence. The foundational connection between Bayes error and total variation is classical \cite{nielsen2014,jiao2017,tao2024}; our contribution is a clinical channel-ceiling framework that adds four prospective and operational results:
\begin{enumerate}
    \item a sharp partial-identification result showing what an observed plateau does, and does not, identify about report noise;
    \item a cross-fitted posterior estimator of the balanced-accuracy ceiling;
    \item an exact reliability--repetition law that predicts how repeated measurements should lift the ceiling; and
    \item a corrected multimodal theorem that characterizes when a new modality strictly improves 0--1 prediction, together with an exact Gaussian complementarity law.
\end{enumerate}

\subsection{Setup: the population ceiling}

Let $Y\in\{0,1\}$, with prevalence $\pi=\Pr(Y=1)$, and let $X$ be the observed baseline feature vector. Denote the class-conditional laws by
\[
P_y = \mathcal{L}(X\mid Y=y), \qquad y\in\{0,1\},
\]
with densities $p_0,p_1$ with respect to a common dominating measure $\mu$. For a measurable classifier $g:\mathcal{X}\to\{0,1\}$, define
\[
\mathrm{BA}(g)=\frac{1}{2}\{\mathrm{TPR}(g)+\mathrm{TNR}(g)\}.
\]
The total variation distance is
\[
\mathrm{TV}(P_0,P_1)
=\frac{1}{2}\int |p_1-p_0|\,d\mu
=\sup_{A}|P_1(A)-P_0(A)|.
\]
We call
\[
\kappa_X := \mathrm{TV}(P_0,P_1)\in[0,1]
\]
the \emph{effective channel separability}. It is the prevalence-invariant amount of class information available for hard classification on the observed channel.

\begin{lemma}[Balanced-accuracy separation identity]
\label{lem:tv-ba}
For any binary prediction problem,
\[
C_{\mathrm{BA}}(X)
:=\sup_g \mathrm{BA}(g)
=\frac{1}{2}\{1+\kappa_X\}.
\]
The optimum is attained by the equal-prior likelihood-ratio rule
\[
g^*(x)=\mathbf{1}\{p_1(x)\ge p_0(x)\}.
\]
For raw 0--1 loss under prevalence $\pi$, the Bayes error is
\[
\begin{aligned}
R_\pi^*(X)
&=\int \min\{\pi p_1,(1-\pi)p_0\}\,d\mu\\
&=\frac{1}{2}-\frac{1}{2}\int
|\pi p_1-(1-\pi)p_0|\,d\mu.
\end{aligned}
\]
\end{lemma}

\begin{proof}
For a decision region $A=\{x:g(x)=1\}$,
\[
\mathrm{BA}(g)
=\frac12+\frac12\{P_1(A)-P_0(A)\}.
\]
Taking the supremum over measurable $A$ gives the total variation distance, attained at $A^*=\{p_1\ge p_0\}$. The raw-risk identity follows by integrating
$\min(a,b)=\tfrac12(a+b-|a-b|)$.
\end{proof}

\begin{lemma}[Data processing and architecture invariance]
\label{lem:data-processing}
Let $T$ be any deterministic or randomized representation computed from $X$, so that $Y\to X\to T$ is a Markov chain. Then
\[
\mathrm{TV}\{\mathcal{L}(T\mid Y=0),\mathcal{L}(T\mid Y=1)\}
\le \kappa_X,
\]
and therefore
\[
C_{\mathrm{BA}}(T)\le C_{\mathrm{BA}}(X),
\qquad
R_\pi^*(T)\ge R_\pi^*(X).
\]
\end{lemma}

\begin{proof}
Total variation contracts under Markov kernels \cite{polyanskiy2015,gaubert2013}. The balanced-accuracy result follows from Lemma~\ref{lem:tv-ba}. For raw risk, every rule based on $T$ is also a rule based on $X$ after composition with the channel $X\mapsto T$, so the admissible rule class based on $X$ is weakly larger.
\end{proof}

Lemma~\ref{lem:data-processing} establishes a common population upper bound for all models operating on the same observed variables. It does \emph{not} imply that finite models must achieve the same performance: approximation, optimization, and estimation errors determine how closely each learner approaches the channel ceiling.

\subsection{Replacement Contamination and General Channels}
We first analyze a transparent report-noise model. Let $U$ denote the informative content that a perfectly functioning instrument would elicit, with class-conditional laws $Q_y=\mathcal{L}(U\mid Y=y)$. The observed report is replaced by class-independent content with probability $\alpha$.

\begin{definition}[Shared replacement-contamination channel]
\label{def:contamination}
For a class-independent probability law $R$ and $\alpha\in[0,1]$,
\[
P_y=(1-\alpha)Q_y+\alpha R,
\qquad y\in\{0,1\}.
\]
\end{definition}

\begin{theorem}[Exact replacement-contamination ceiling]
\label{thm:contam}
Let $\tau=\TV(Q_0,Q_1)$. Then
\[
\kappa_X=(1-\alpha)\tau,
\qquad
\cba(X)=\tfrac12\{1+(1-\alpha)\tau\}\le1-\alpha/2.
\]
Equality in the upper bound holds iff $\tau=1$. Moreover, $\Rstar_\pi(X)\ge\alpha\min\{\pi,1-\pi\}$.
\end{theorem}
\begin{proof}
The shared contamination component cancels:
\[
p_1-p_0=(1-\alpha)(q_1-q_0).
\]
Taking the total variation norm gives $\kappa_X=(1-\alpha)\tau$, and Lemma~\ref{lem:ba} gives the balanced-accuracy ceiling. For raw risk, pointwise,
\[
\min\{\pi p_1,(1-\pi)p_0\}
\ge \alpha r\min\{\pi,1-\pi\}.
\]
Integrating proves the final inequality.
\end{proof}
\begin{proposition}[Sharp identified set]
\label{prop:id}
If the model above holds and the population balanced-accuracy ceiling is $c\in[1/2,1]$, then without external knowledge of $\tau$ the sharp identified set is
\[
\mathcal I_\alpha(c)=[0,2(1-c)].
\]
Thus $c=0.85$ identifies $\kappa_X=0.70$ and only $\alpha\in[0,0.30]$.
\end{proposition}
\begin{proof}
The observation $c$ identifies only
\[
\kappa_X=2c-1=(1-\alpha)\tau.
\]
Because $0\le\tau\le1$, necessarily $\alpha\le1-\kappa_X=2(1-c)$. Conversely, for any $\alpha$ in this interval, choosing
\[
\tau=\frac{\kappa_X}{1-\alpha}\le1
\]
reproduces the same observed ceiling. Hence the interval is sharp.
\end{proof}

\subsection{General report channels}

The exact contamination model is a special case of information contraction. Let $W(dx\mid u)$ be a common report channel mapping latent content $U$ to an observed report $X$, and define its Dobrushin coefficient
\[
\vartheta(W)
:=\sup_{u,u'}\mathrm{TV}\{W(\cdot\mid u),W(\cdot\mid u')\}.
\]

\begin{theorem}[Dobrushin channel bound]
\label{thm:dobrushin}
If $X\sim W(Q_y)$ conditional on $Y=y$, then
\[
\kappa_X
\le \vartheta(W)\,\mathrm{TV}(Q_0,Q_1),
\]
and
\[
C_{\mathrm{BA}}(X)
\le \frac12\left\{1+\vartheta(W)\mathrm{TV}(Q_0,Q_1)\right\}.
\]
The shared replacement channel in Definition~\ref{def:contamination} has $\vartheta(W)=1-\alpha$.
\end{theorem}

\begin{proof}
This is the strong data-processing inequality for total variation \cite{polyanskiy2015,gaubert2013}. For the replacement kernel
$W(\cdot\mid u)=(1-\alpha)\delta_u+\alpha R$, the shared $R$ term cancels between two inputs, giving Dobrushin coefficient $1-\alpha$.
\end{proof}

The common-channel assumption is substantive. If reporting behavior depends directly on disease status after conditioning on $U$, then the channel is $W_y(dx\mid u)$ rather than a shared $W$, and differential reporting bias may either destroy or create apparent class separation. Such violations must be examined empirically rather than absorbed into a single scalar noise parameter.

\subsection{An operational cross-fitted ceiling estimator}
\label{sec:ceiling-estimator}

The preceding results describe the population ceiling. To make the theory testable, we express the ceiling through the equal-prior posterior. Let
\[
M=\frac12(P_0+P_1)
\]
and define
\[
\eta_{\mathrm{eq}}(x)
:=\Pr_{M}(Y=1\mid X=x)
=\frac{p_1(x)}{p_0(x)+p_1(x)}.
\]

\begin{proposition}[Posterior representation of channel separability]
\label{prop:posterior-tv}
The effective channel separability satisfies
\[
\kappa_X
=\mathbb{E}_{X\sim M}\left|2\eta_{\mathrm{eq}}(X)-1\right|,
\]
and therefore
\[
C_{\mathrm{BA}}(X)
=\frac12\left[1+
\mathbb{E}_{X\sim M}\left|2\eta_{\mathrm{eq}}(X)-1\right|
\right].
\]
\end{proposition}

\begin{proof}
Because $M$ has density $m=(p_0+p_1)/2$,
\[
\begin{aligned}
\mathbb{E}_{M}|2\eta_{\mathrm{eq}}(X)-1|
&=\int \frac{|p_1-p_0|}{p_0+p_1}\frac{p_0+p_1}{2}\,d\mu\\
&=\frac12\int|p_1-p_0|\,d\mu
=\mathrm{TV}(P_0,P_1).
\end{aligned}
\]
\end{proof}

Proposition~\ref{prop:posterior-tv} motivates a \emph{cross-fitted channel-ceiling estimator}. Construct a balanced evaluation sample, partition it into $K$ folds, estimate the equal-prior posterior on the other $K-1$ folds, and obtain out-of-fold predictions $\widehat\eta_{-k(i)}(X_i)$. Define
\[
\begin{aligned}
\widehat\kappa_{\mathrm{CF}}
&=\frac1n\sum_{i=1}^n
\left|2\widehat\eta_{-k(i)}(X_i)-1\right|,\\
\widehat C_{\mathrm{BA,CF}}
&=\frac12(1+\widehat\kappa_{\mathrm{CF}}).
\end{aligned}
\]
Cross-fitting prevents the trivial optimism that would arise from evaluating a high-capacity posterior model on its training observations.

\begin{proposition}[Consistency of the cross-fitted ceiling estimator]
\label{prop:cf-consistency}
Assume that each out-of-fold posterior estimator is $L^1(M)$-consistent for $\eta_{\mathrm{eq}}$ and that fold sizes diverge. Then
\[
\widehat\kappa_{\mathrm{CF}}\xrightarrow{p}\kappa_X,
\qquad
\widehat C_{\mathrm{BA,CF}}\xrightarrow{p}C_{\mathrm{BA}}(X).
\]
\end{proposition}

\begin{proof}
The map $a\mapsto|2a-1|$ is $2$-Lipschitz on $[0,1]$. Hence the difference between the plug-in integrand and its population target is bounded by $2|\widehat\eta-\eta_{\mathrm{eq}}|$. $L^1$ consistency controls this term, while the out-of-fold empirical average converges to its expectation by the law of large numbers.
\end{proof}

In finite samples, posterior underfitting typically biases the estimate toward $1/2$ and therefore underestimates the ceiling. We therefore recommend reporting estimates from multiple flexible posterior learners, nested cross-validation, and bootstrap confidence intervals. This estimator is a ceiling diagnostic, not a proof that a particular trained model has reached Bayes optimality.

\subsection{Reliability and repeated measurements: a prospective ceiling law}

The contamination parameter $\alpha$ should not be equated with Cronbach's alpha or generic instrument reliability. A rigorous reliability prediction requires an explicit measurement model. Consider an additive Gaussian true-score channel with repeated administrations.

\begin{assumption}[Gaussian repeated-measurement channel]
\label{ass:gaussian-repeats}
For $y\in\{0,1\}$,
\[
U\mid Y=y\sim\mathcal{N}(\mu_y,\sigma_U^2),
\]
and repeated observations satisfy
\[
X_j=U+\varepsilon_j,
\qquad
\varepsilon_j\stackrel{\mathrm{iid}}{\sim}\mathcal{N}(0,\sigma_\varepsilon^2),
\qquad j=1,\ldots,m,
\]
with the errors independent of $(U,Y)$. Define the within-class single-measure reliability
\[
\rho=\frac{\sigma_U^2}{\sigma_U^2+\sigma_\varepsilon^2}
\]
and the latent standardized separation
\[
d'_U=\frac{|\mu_1-\mu_0|}{\sigma_U}.
\]
\end{assumption}

\begin{theorem}[Reliability--repetition ceiling law]
\label{thm:repeated-measurements}
Under Assumption~\ref{ass:gaussian-repeats}, the average report
\[
\overline X_m=\frac1m\sum_{j=1}^m X_j
\]
has effective standardized separation
\[
d'_m
=d'_U\sqrt{\frac{m\rho}{1+(m-1)\rho}},
\]
and exact balanced-accuracy ceiling
\[
\boxed{
C_{\mathrm{BA}}(\overline X_m)
=\Phi\left(
\frac{d'_U}{2}
\sqrt{\frac{m\rho}{1+(m-1)\rho}}
\right),
}
\]
where $\Phi$ is the standard normal cumulative distribution function. For $d'_U>0$, the ceiling is strictly increasing in $\rho$, strictly increasing in $m$ when $\rho<1$, and
\[
\lim_{m\to\infty}C_{\mathrm{BA}}(\overline X_m)
=\Phi(d'_U/2),
\]
the latent-score ceiling.
\end{theorem}

\begin{proof}
Conditional on $Y=y$,
\[
\overline X_m\sim
\mathcal{N}\left(\mu_y,\sigma_U^2+\frac{\sigma_\varepsilon^2}{m}\right).
\]
Therefore
\[
d'_m
=\frac{|\mu_1-\mu_0|}
{\sqrt{\sigma_U^2+\sigma_\varepsilon^2/m}}
=d'_U\left(1+\frac{1-\rho}{m\rho}\right)^{-1/2},
\]
which is algebraically equivalent to the stated expression. For equal-variance Gaussian classes, the equal-prior Bayes rule thresholds at the midpoint and attains balanced accuracy $\Phi(d'_m/2)$. Monotonicity and the limit follow directly.
\end{proof}

Theorem~\ref{thm:repeated-measurements} supplies a forward prediction that can be tested without fitting increasingly large models: if measurement error is an important source of the plateau, repeated administrations should improve the ceiling along a saturating curve whose shape is determined by independently estimated within-class reliability. Reliability alone does not determine the ceiling; the latent relevance $d'_U$ is also necessary.

\subsection{When does an additional modality improve prediction?}

Let $Z$ be an auxiliary modality. The joint observation can always ignore $Z$, so it cannot have worse Bayes performance. Strict improvement, however, requires more than positive conditional mutual information.

\begin{theorem}[Multimodal non-decrease and exact strictness criterion]
\label{thm:multimodal-strictness}
For balanced accuracy,
\[
C_{\mathrm{BA}}(X,Z)\ge C_{\mathrm{BA}}(X),
\]
with strict inequality if and only if
\[
\mathrm{TV}(P_0^{XZ},P_1^{XZ})
>
\mathrm{TV}(P_0^{X},P_1^{X}).
\]
For raw 0--1 loss, let
\[
\begin{aligned}
\eta(X,Z)&=\Pr(Y=1\mid X,Z),\\
\eta_X(X)&=\Pr(Y=1\mid X).
\end{aligned}
\]
Then
\[
\begin{aligned}
R^*(X)-R^*(X,Z)
&=\mathbb{E}\left|\eta(X,Z)-\frac12\right|\\
&\quad-\mathbb{E}\left|\eta_X(X)-\frac12\right|\ge0.
\end{aligned}
\]
Equality holds if and only if, for $P_X$-almost every $x$, the conditional random variable $\eta(x,Z)-1/2$ does not change sign almost surely. Consequently, strict raw-accuracy improvement occurs precisely when the new modality changes the Bayes-optimal class decision on a set of positive probability.
\end{theorem}

\begin{proof}
The balanced-accuracy statement follows because marginalization $(X,Z)\mapsto X$ is a Markov kernel, so total variation of the marginal cannot exceed total variation of the joint law. For raw risk,
\[
R^*(X)=\frac12-\mathbb{E}|\eta_X(X)-1/2|
\]
and similarly for $(X,Z)$. Since
\[
\eta_X(X)=\mathbb{E}\{\eta(X,Z)\mid X\},
\]
conditional Jensen's inequality for the convex absolute-value function yields the non-negativity. Equality in Jensen's inequality for $|\cdot|$ holds exactly when the conditional support remains within one of its affine regions, namely one side of zero.
\end{proof}

The criterion clarifies why $I(Z;Y\mid X)>0$ is insufficient for strict accuracy gain: $Z$ may refine posterior confidence while leaving every posterior on the same side of the decision threshold. Such a modality may improve log loss, calibration, or ranking without improving 0--1 accuracy.

For a tractable exact ceiling under complementary modalities, consider conditionally independent Gaussian measurements.

\begin{proposition}[Gaussian multimodal complementarity law]
\label{prop:gaussian-multimodal}
Let $X^{(1)},\ldots,X^{(M)}$ be conditionally independent given $Y$, with
\[
X^{(j)}\mid Y=y
\sim\mathcal{N}(\mu_{jy},\Sigma_j),
\]
where each covariance $\Sigma_j$ is common across classes and positive definite. Define
\[
d_j^2
=(\mu_{j1}-\mu_{j0})^\top
\Sigma_j^{-1}
(\mu_{j1}-\mu_{j0}).
\]
Then the concatenated observation has
\[
d_{\mathrm{joint}}^2=\sum_{j=1}^{M}d_j^2
\]
and exact balanced-accuracy ceiling
\[
\boxed{
C_{\mathrm{BA}}(X^{(1)},\ldots,X^{(M)})
=
\Phi\left(
\frac12\sqrt{\sum_{j=1}^{M}d_j^2}
\right).
}
\]
If a newly added modality has $d_j>0$, it strictly raises the joint ceiling unless the existing feature set already has perfect separation.
\end{proposition}

\begin{proof}
Conditional independence makes the joint covariance block diagonal. The squared Mahalanobis separation of the concatenated Gaussian vector is therefore the sum of the blockwise squared separations. Equal-covariance Gaussian discrimination has Bayes balanced accuracy $\Phi(d/2)$.
\end{proof}

Measurement noise is incorporated through $\Sigma_j$: a cleaner channel reduces the observed covariance relative to the class mean difference and increases $d_j$. Proposition~\ref{prop:gaussian-multimodal} also corrects an important interpretive point: in general,
\[
C_{\mathrm{BA}}(X,Z)
\ge \max\{C_{\mathrm{BA}}(X),C_{\mathrm{BA}}(Z)\},
\]
but the right-hand side is only a lower bound on the joint ceiling. Complementary modalities can yield a joint ceiling strictly above both single-modality ceilings.

\section{Cross-Fitted Frontier Audit}
\subsection{Equal-Prior Target and Importance Weighting}
The population identity uses the equal-prior mixture $M=\tfrac12(P_0+P_1)$, not the prevalence-weighted marginal $P_X=(1-\pi)P_0+\pi P_1$. If a probabilistic learner estimates the ordinary posterior $\eta_\pi(x)=P(Y=1\mid X=x)$, the equal-prior posterior is
\[
\eta_{\rm eq}(x)=
\frac{(1-\pi)\eta_\pi(x)}{(1-\pi)\eta_\pi(x)+\pi\{1-\eta_\pi(x)\}}.
\]
Equivalently, training with balanced class weights targets the equal-prior decision problem directly. For any integrable $h$,
\[
\mathbb E_M h(X)=\tfrac12\mathbb E\{h(X)\mid Y=0\}
+\tfrac12\mathbb E\{h(X)\mid Y=1\}
=\mathbb E_P[w(Y)h(X)],
\]
where $w(0)=1/[2(1-\pi)]$ and $w(1)=1/(2\pi)$. In a finite evaluation fold, the numerically stable equivalent is the class-normalized average
\[
\widehat\kappa=\frac12\left[
\frac1{n_0}\sum_{i:Y_i=0}|2\widehat\eta_{{\rm eq},i}-1|
+\frac1{n_1}\sum_{i:Y_i=1}|2\widehat\eta_{{\rm eq},i}-1|
\right].
\]
This weighting is required when the cohort is not artificially balanced.

\subsection{Algorithm}
\begin{algorithm}[H]
\caption{Cross-fitted channel-frontier audit}
\begin{algorithmic}[1]
\REQUIRE observations $(X_i,Y_i)$; optional group identifier; probabilistic learner; folds; bootstrap scheme
\STATE Construct stratified folds, respecting groups when repeated observations belong to one patient.
\FOR{each fold $k$}
  \STATE Fit the probabilistic learner on all other folds with balanced class weighting or posterior transformation.
  \STATE Store out-of-fold $\widehat\eta_{{\rm eq},i}$ for observations in fold $k$.
\ENDFOR
\STATE Compute $\widehat\cba_{\rm CF}=\tfrac12(1+\widehat\kappa)$ using class-normalized evaluation.
\STATE Tune the balanced-accuracy decision threshold using training data only; compute achieved out-of-fold BA and $G=\widehat\cba_{\rm CF}-\BA_{\rm achieved}$.
\STATE Bootstrap the independent sampling unit (patients for UCI; respondents for BRFSS and NHANES) and repeat the complete audit.
\STATE Run the permutation-null and training-fraction diagnostics below.
\RETURN frontier estimate, interval, achieved BA, AUROC, learner gap, optimism floor, and underfit verdict.
\end{algorithmic}
\end{algorithm}

\subsection{Bias Diagnostics}
\paragraph{Permutation-null optimism floor.}
After shuffling labels within the valid sampling structure, the population class-conditional laws coincide and the true ceiling is $0.5$. Define
\[
\Delta_{\rm null}=\widehat\cba_{\rm perm}-0.5.
\]
A positive value measures finite-sample overconfidence of the posterior plug-in functional. It is a diagnostic, not a correction term: subtracting it need not remove bias under the original signal distribution.

\paragraph{Underfit curve.}
The posterior learner is refit at fractions $0.25/0.5/0.75/1.0$ of the available training data while preserving the evaluation protocol. A materially positive final increment indicates that the estimate is still increasing and should be reported as a lower bound. A small terminal change or a non-monotone oscillation within sampling noise is treated as convergence. Population monotonicity under added variables does not imply finite-sample monotonicity: an uninformative coordinate can reduce estimation efficiency and make a joint estimate slightly lower than a marginal estimate.

\begin{figure*}[t]
\centering
\includegraphics[width=.72\textwidth]{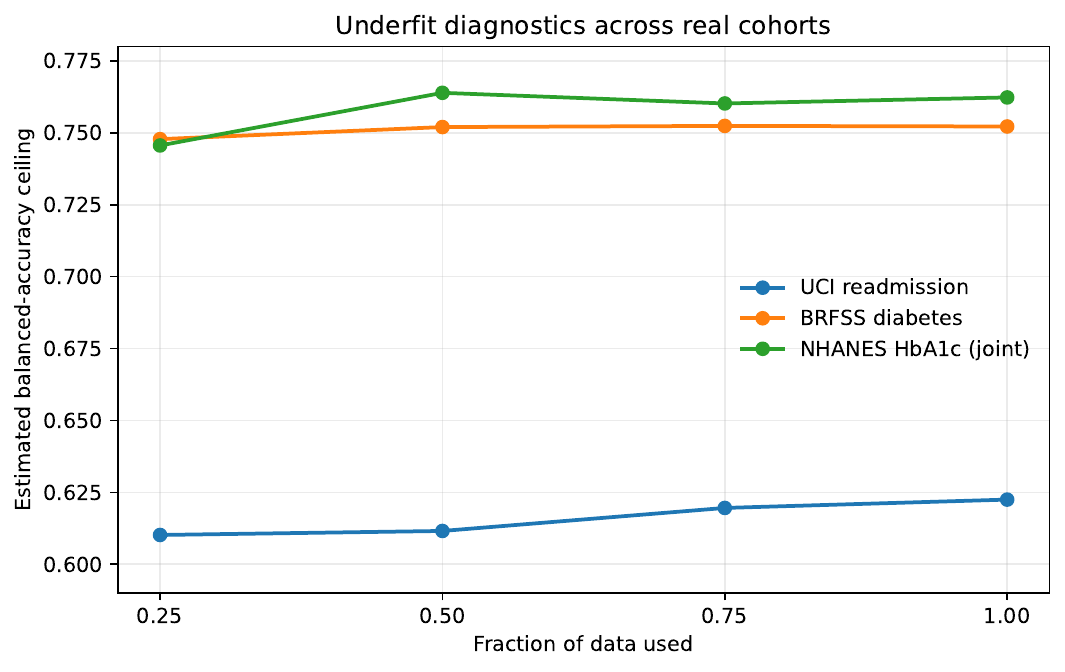}
\caption{Training-fraction diagnostic. UCI is still rising and is reported as a lower bound. BRFSS stabilizes. NHANES oscillates non-monotonically and is treated as converged rather than rising.}
\label{fig:supp-underfit}
\end{figure*}

\section{Real-Cohort Experimental Details}
\label{sec:real-cohorts}
\subsection{UCI Diabetes 130-US Hospitals Readmission}
The source cohort is the UCI diabetes hospital dataset described by Strack et al. \cite{strack2014impact}. The raw file contains $101{,}766$ encounters. Removing death and hospice discharges leaves $99{,}343$ encounters from $69{,}990$ unique patients, with 30-day-readmission prevalence $0.1139$. Folds use \texttt{StratifiedGroupKFold} with \texttt{patient\_nbr}; bootstrap replicates resample patients, not encounters. Administrative channel $A$ and clinical channel $B$ follow the experiment's feature partition. 

The audit gives ceiling $0.6225$ with 95\% CI $[0.6213,0.6240]$, AUROC $0.6692$, achieved BA $0.6223$, and $G=+0.0002$. The permutation-null ceiling is $0.5147$, so the optimism floor is $+0.0147$. The underfit sequence is $0.6102,0.6116,0.6196,0.6225$; its final $+0.0029$ increment requires the lower-bound label. Channel ceilings are $0.5753/0.5961/0.6206$ for $A/B/A+B$; AUROC changes $0.6058\to0.6674$, flip rate is $0.2996$, and risk gain is $0.0454$.

The full learner decomposition is reported in Table~\ref{tab:uci-learners} of Section~\ref{sec:learner-panels}. It shows that the near-zero GBDT gap is not forced by the estimator: LR has gap $+0.0039$, RF has $+0.0861$, and MLP has $+0.0527$. In particular, RF attains AUROC $0.6648$ but only BA $0.5349$, demonstrating that a learner may rank observations reasonably while remaining far below the best thresholded rule supported by its inputs.

\begin{figure*}[t]
\centering
\includegraphics[width=.72\textwidth]{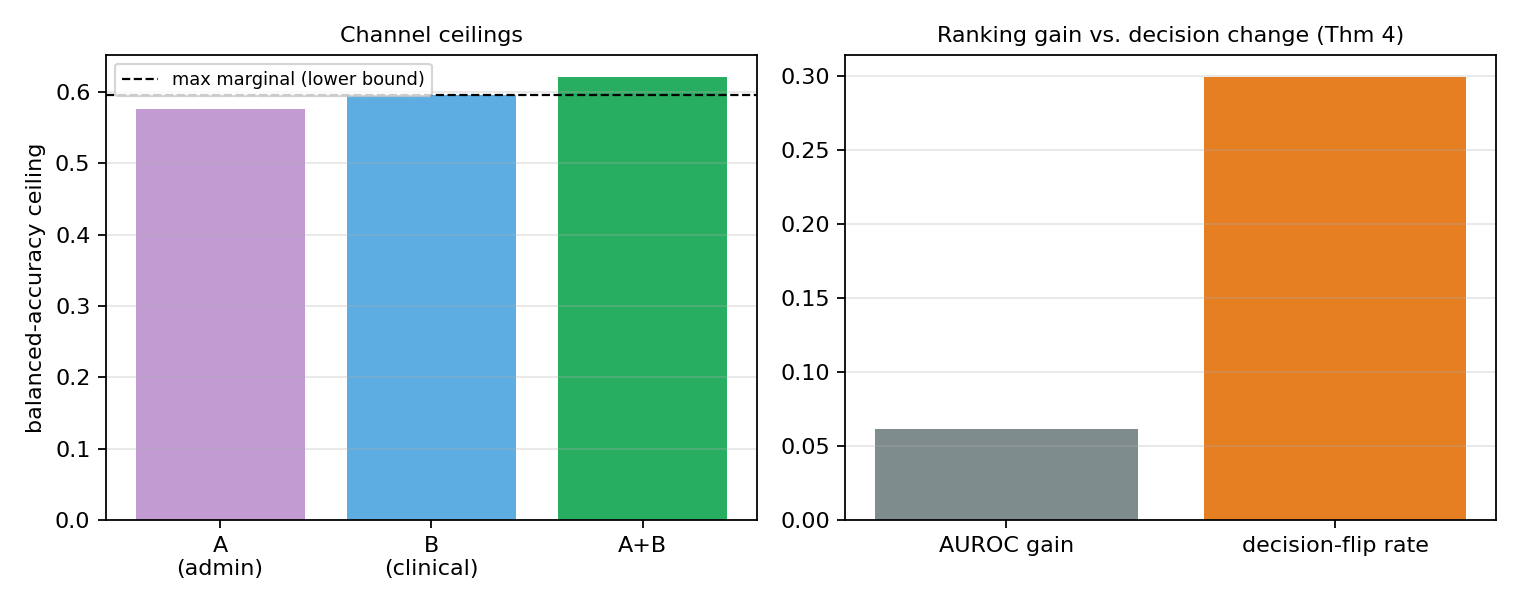}
\caption{UCI administrative, clinical, and joint channel results, together with the ranking gain and decision-flip rate.}
\label{fig:uci-channel}
\end{figure*}

\subsection{CDC BRFSS 2015 Diabetes Indicators}
The BRFSS analysis uses $253{,}680$ respondents and prevalence $0.1393$ \cite{cdcbrfss2015}. Every predictor is obtained by telephone survey, so this is the purest report-channel cohort in the paper. There is one row per respondent, and no grouping variable is needed. Theorem~4 partitions a perception channel $A=\{$GenHlth, MentHlth, PhysHlth, DiffWalk$\}$ and a recalled-diagnosis channel $B=\{$HighBP, HighChol, CholCheck, Stroke, HeartDiseaseorAttack, BMI$\}$. Channel $B$ represents prior objective measurements transmitted through memory and survey response.

The overall ceiling is $0.7522$ $[0.7515,0.7529]$, AUROC $0.8298$, achieved BA $0.7518$, and $G=+0.0003$. The optimism floor is $+0.0045$. The underfit sequence $0.7478,0.7520,0.7524,0.7522$ ends at $-0.0002$ and is converged. Channel ceilings $0.6917/0.7165/0.7413$ imply complementarity $+0.0248$; AUROC changes $0.7488\to0.8160$, flip rate is $0.1907$, and risk gain is $0.0496$.

Table~\ref{tab:brfss-learners} provides the corresponding learner panel. GBDT nearly reaches the estimated frontier ($G=+0.0003$), whereas the MLP has competitive AUROC $0.8180$ but BA $0.5793$, leaving gap $+0.1642$. The contrast connects the cohort-level result to the paper's ranking--decision distinction: strong ordering alone does not guarantee a useful hard decision rule. 

\subsection{NHANES 2015--2018 HbA1c}
NHANES contributes $10{,}219$ adults age $\ge20$ with glycohemoglobin measured; prevalence of HbA1c $\ge6.5\%$ is $0.1409$ \cite{nhanes20152016,nhanes20172018}. The questionnaire channel is
\begin{quote}\small
RIDAGEYR, RIAGENDR, RIDRETH3, DMDEDUC2, INDFMPIR, HUQ010, SMQ020, PAQ650.
\end{quote}
It is called \emph{questionnaire}, not self-report or PROM: it is mostly demographic, with self-rated health and two behavior variables. The measured channel is
\begin{quote}\small
BMXBMI, BMXWAIST, BMXHT, BMXWT, LBDHDD, LBXTC, LBXSATSI, LBXSASSI, LBXSAL, LBXSCR, LBXSUA, LBXSTR, LBXSGTSI, LBXWBCSI, LBXRBCSI, LBXHGB, LBXPLTSI.
\end{quote}
An explicit assertion excludes LBXGH, LBXGLU, LBXSGL, and LBXIN from both channels, preventing glycemic outcome leakage.

The regularized posterior learner is HistGradientBoosting with early stopping and native NaN handling (no sentinel imputation). Its settings are \texttt{max\_leaf\_nodes=15}, \texttt{min\_samples\_leaf=50}, $\ell_2=1.0$, and learning rate $0.05$. Questionnaire, measured, and joint ceilings are $0.7110$ $[0.7082,0.7141]$, $0.7152$ $[0.7118,0.7190]$, and $0.7623$ $[0.7582,0.7659]$. Their achieved BAs are $0.7093$, $0.7063$, and $0.7541$; AUROCs are $0.7714$, $0.7854$, and $0.8419$. The marginal gap $B-A=+0.0042$ is null because the intervals overlap. Complementarity $A+B-\max(A,B)=+0.0471$ is significant because the joint interval is disjoint from the measured interval. Flip rate is $0.2105$ and risk gain is $0.0513$. The optimism floor is $+0.0309$; the underfit sequence $0.7456,0.7639,0.7602,0.7623$ drops $0.0037$ from $0.5$ to $0.75$ and is treated as finite-sample oscillation.

The joint-channel learner results appear in Table~\ref{tab:nhanes-learners}. The panel again separates frontier estimation from learner quality: LR, RF, and GBDT have gaps $+0.0035$, $+0.0196$, and $+0.0081$, while the MLP reaches AUROC $0.8385$ but BA $0.5724$, producing gap $+0.1933$. Together with the null marginal channel contrast and significant joint complementarity, this shows that channel value is determined by conditional decision information rather than by the labels ``questionnaire or ``measured. 


\subsection{Complete Bootstrap Intervals}
\begin{table}[H]
\centering
\small
\caption{Available 95\% bootstrap intervals for channel frontiers.}
\begin{tabular}{llc}
\toprule
Cohort & Channel & Ceiling (95\% CI)\\
\midrule
UCI & administrative + clinical & $0.6225$ $[0.6213,0.6240]$\\
BRFSS & all survey predictors & $0.7522$ $[0.7515,0.7529]$\\
NHANES & questionnaire & $0.7110$ $[0.7082,0.7141]$\\
NHANES & measured & $0.7152$ $[0.7118,0.7190]$\\
NHANES & questionnaire + measured & $0.7623$ $[0.7582,0.7659]$\\
\bottomrule
\end{tabular}
\end{table}

\section{Complete Learner Panels}
\label{sec:learner-panels}
\begin{table}[H]
\centering
\small
\caption{UCI learner panel. Gap is ceiling minus achieved balanced accuracy.}
\label{tab:uci-learners}
\begin{tabular}{lrrrr}
\toprule
Learner & BA & AUROC & Ceiling & Gap\\
\midrule
LR & $0.6016$ & $0.6468$ & $0.6055$ & $+0.0039$\\
RF & $0.5349$ & $0.6648$ & $0.6211$ & $+0.0861$\\
GBDT & $0.6223$ & $0.6692$ & $0.6225$ & $+0.0002$\\
MLP & $0.5146$ & $0.5925$ & $0.5673$ & $+0.0527$\\
\bottomrule
\end{tabular}
\end{table}

\begin{table}[H]
\centering
\small
\caption{BRFSS learner panel.}
\label{tab:brfss-learners}
\begin{tabular}{lrrrr}
\toprule
Learner & BA & AUROC & Ceiling & Gap\\
\midrule
LR & $0.7461$ & $0.8225$ & $0.7481$ & $+0.0020$\\
RF & $0.7269$ & $0.8239$ & $0.7478$ & $+0.0209$\\
GBDT & $0.7518$ & $0.8298$ & $0.7522$ & $+0.0003$\\
MLP & $0.5793$ & $0.8180$ & $0.7435$ & $+0.1642$\\
\bottomrule
\end{tabular}
\end{table}

\begin{table}[H]
\centering
\small
\caption{NHANES joint-channel learner panel.}
\label{tab:nhanes-learners}
\begin{tabular}{lrrrr}
\toprule
Learner & BA & AUROC & Ceiling & Gap\\
\midrule
LR & $0.7426$ & $0.8151$ & $0.7461$ & $+0.0035$\\
RF & $0.7358$ & $0.8341$ & $0.7554$ & $+0.0196$\\
GBDT & $0.7541$ & $0.8419$ & $0.7623$ & $+0.0081$\\
MLP & $0.5724$ & $0.8385$ & $0.7656$ & $+0.1933$\\
\bottomrule
\end{tabular}
\end{table}



\section{Clinical Evidence-Synthesis Protocol}
\label{sec:evidence-synthesis}
The three real-cohort audits establish that the proposed frontier diagnostic can distinguish learner deficiency from measurement limitation in specific tasks. They do not, by themselves, show whether similar saturation patterns recur across diseases, outcomes, institutions, and learner families. We therefore complement the cohort experiments with a PRISMA-guided evidence synthesis designed to answer three broader questions: whether reported clinical prediction performance repeatedly occupies a restricted range; whether increasing sample size or model complexity systematically moves that range; and under what measurement configurations performance exceeds it. The review is used as descriptive external context rather than as a pooled estimate of a universal ceiling, because the source literature reports heterogeneous outcomes, validation schemes, populations, and metrics. In particular, AUROC, accuracy, balanced accuracy, F1, and AUPRC are retained on their original scales rather than combined into a common estimand.

\subsection{Review Question, Search, and Eligibility}
The synthesis follows PRISMA 2020 and umbrella-review guidance \cite{page2021prisma,aromataris2015summarizing}. Searches covered PubMed, PubMed Central, ScienceDirect, SpringerLink, Authorea, and arXiv. The PubMed string was:
\begin{quote}\small
(``machine learning''[MeSH] OR ``deep learning''[tiab] OR ``random forest''[tiab] OR ``gradient boosting''[tiab] OR ``neural network''[tiab]) AND (``clinical outcome''[tiab] OR ``surgical outcome''[tiab] OR ``patient-reported outcome''[tiab] OR ``PROM''[tiab] OR ``treatment response''[tiab]) AND (``systematic review''[pt] OR ``meta-analysis''[pt] OR ``accuracy''[tiab] OR ``AUC''[tiab]).
\end{quote}
Eligible reports were English-language systematic, scoping, or meta-analytic reviews, plus large primary studies with $n\ge500$, applying ML to structured clinical or patient-reported inputs and reporting a quantitative predictive metric. Imaging-only reports were excluded from the core structured-record synthesis but retained as channel comparators.

\subsection{Screening Units and Analytic Units}
The supplied screening record reports $1{,}117$ database records and 19 manually identified records, $1{,}016$ after duplicate removal, 768 title/abstract exclusions, 248 full-text assessments, and 144 reported full-text exclusions. Separately, the analytic dataset contains 30 source publications and 104 extracted task-level observations across more than 18 categories. Figure~\ref{fig:prisma-units} labels these as distinct units instead of making 30 sources and 104 task observations appear to be the same denominator.

\begin{figure*}[t]
\centering
\includegraphics[width=.88\textwidth]{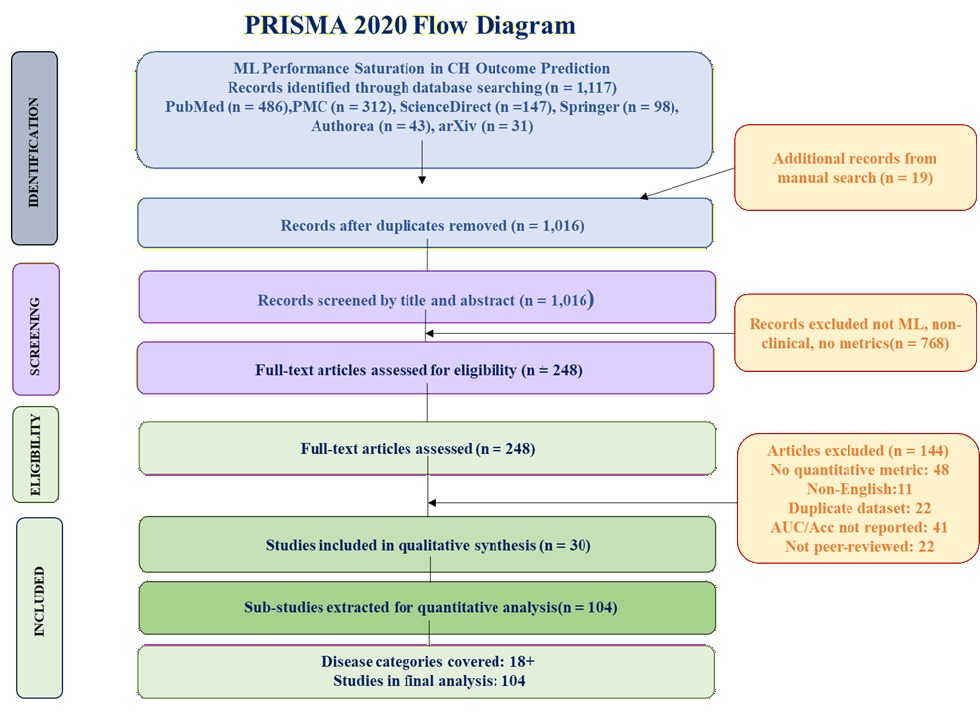}
\caption{PRISMA 2020 Flow Diagram. From 1,117 records identified across six databases, 30 reviews and primary
studies were included, yielding 104 sub-study observations across 18+ disease categories.}
\label{fig:prisma-units}
\end{figure*}

\subsection{Extraction, Metric Policy, and Risk of Bias}
Extracted fields were disease category, learner family, sample size, reported AUROC/accuracy/F1, validation design, class-specific recall when available, and multimodal status. AUROC, balanced accuracy, raw accuracy, F1, and AUPRC are not pooled as one estimand. The main paper uses the single-index AUROC-to-BA conversion only to display a contextual band, never to convert individual studies into audited channel frontiers. Risk of bias was organized around sample adequacy, reporting completeness, and validation approach, adapted from ROBIS and PROBAST.

\section{Result Analysis}
\label{sec:result-analysis}
\subsection{Disease-Category Summary}

Table~\ref{tab:disease-summary} summarizes 104 task-level observations extracted from 30 source publications and spanning more than 18 disease categories. The evidence base is deliberately broad: it includes surgical outcomes, cardiovascular disease, stroke, endocrine and renal disease, obstetrics, oncology, liver transplantation, mental health, autoimmune disease, hospital readmission, chronic pain, medical imaging, and ECG-based prediction. The largest task groups are orthopedic surgery/PROM prediction (12 observations), endocrine/renal/diabetes (11), and general oncology (11), followed by cardiovascular disease and ICU/sepsis (8 each). This breadth is useful because the same qualitative question---whether performance is limited by the learner or by the recorded channel---appears across very different clinical endpoints.

The table also makes clear why the synthesis is descriptive. A ``disease category'' may contain prognosis, treatment response, diagnosis, complications, or quality-of-life outcomes, and the reported ranges may combine AUROC and accuracy. Consequently, category ranges should not be read as pooled effect estimates or as directly comparable channel frontiers. Their value is pattern discovery: many structured-clinical categories repeatedly occupy a middle performance region, while chronic pain and readmission provide lower examples and imaging/ECG provide higher-signal comparators. The wide ranges for ICU/sepsis, breast cancer, autoimmune disease, and Parkinson disease further indicate that measurement composition and validation design vary substantially within a nominal disease label.

\begin{table*}[t]
\centering
\scriptsize
\caption{Reported category ranges from the supplied evidence table. Values remain descriptive and may combine AUROC and accuracy.}
\label{tab:disease-summary}
\begin{tabular}{lrrll}
\toprule
Category & Tasks & Range & Interpretation & Frequent model\\
\midrule
Medical imaging & 5 & $.85$--$.97$ & objective-signal comparator & CNN\\
ECG/cardiac signal & 2 & $.83$--$.99$ & objective-signal comparator & DL\\
Parkinson disease & 3 & $.75$--$.92$ & clinical and multimodal & RF/GenoML\\
Dementia/Alzheimer disease & 3 & $.75$--$.92$ & highly heterogeneous modalities & RF/CNN/XGB\\
Cardiac surgery & 6 & $.80$--$.90$ & recurrent region & RF/XGB\\
Stroke & 4 & $.80$--$.92$ & recurrent to above & XGB\\
Cardiovascular disease & 8 & $.80$--$.92$ & typical near $.85$ & XGB\\
Orthopedic surgery/PROM & 12 & $.80$--$.88$ & recurrent region & RF/XGB\\
ICU/sepsis & 8 & $.75$--$.99$ & broad range & XGB/LSTM\\
Endocrine/renal/diabetes & 11 & $.78$--$.90$ & typical $.80$--$.87$ & RF/XGB\\
Obstetrics & 6 & $.78$--$.90$ & typical $.78$--$.88$ & XGB/RF\\
General oncology & 11 & $.72$--$.93$ & typical $.78$--$.85$ & RF/NN/LASSO\\
Liver/transplant & 2 & $.75$--$.92$ & typical $.82$--$.88$ & RF/GBM\\
Mental health/depression & 3 & $.70$--$.92$ & typical $.78$--$.85$ & RF/XGB/LR\\
Breast cancer & 3 & $.57$--$.97$ & clinical lower; imaging higher & CNN/RF/SVM\\
Autoimmune/rheumatology & 3 & $.63$--$.92$ & clinical $.75$--$.85$ & RF/SVM/XGB\\
Hospital readmission & 3 & $.65$--$.82$ & below recurrent region & XGB\\
Chronic pain/PROM & 2 & $.49$--$.86$ & lowest PROM-heavy category & RF/SVM/LR\\
\bottomrule
\end{tabular}
\end{table*}

\subsection{Disease-Category Performance}
\begin{figure*}[t]
\centering
\includegraphics[width=.95\textwidth]{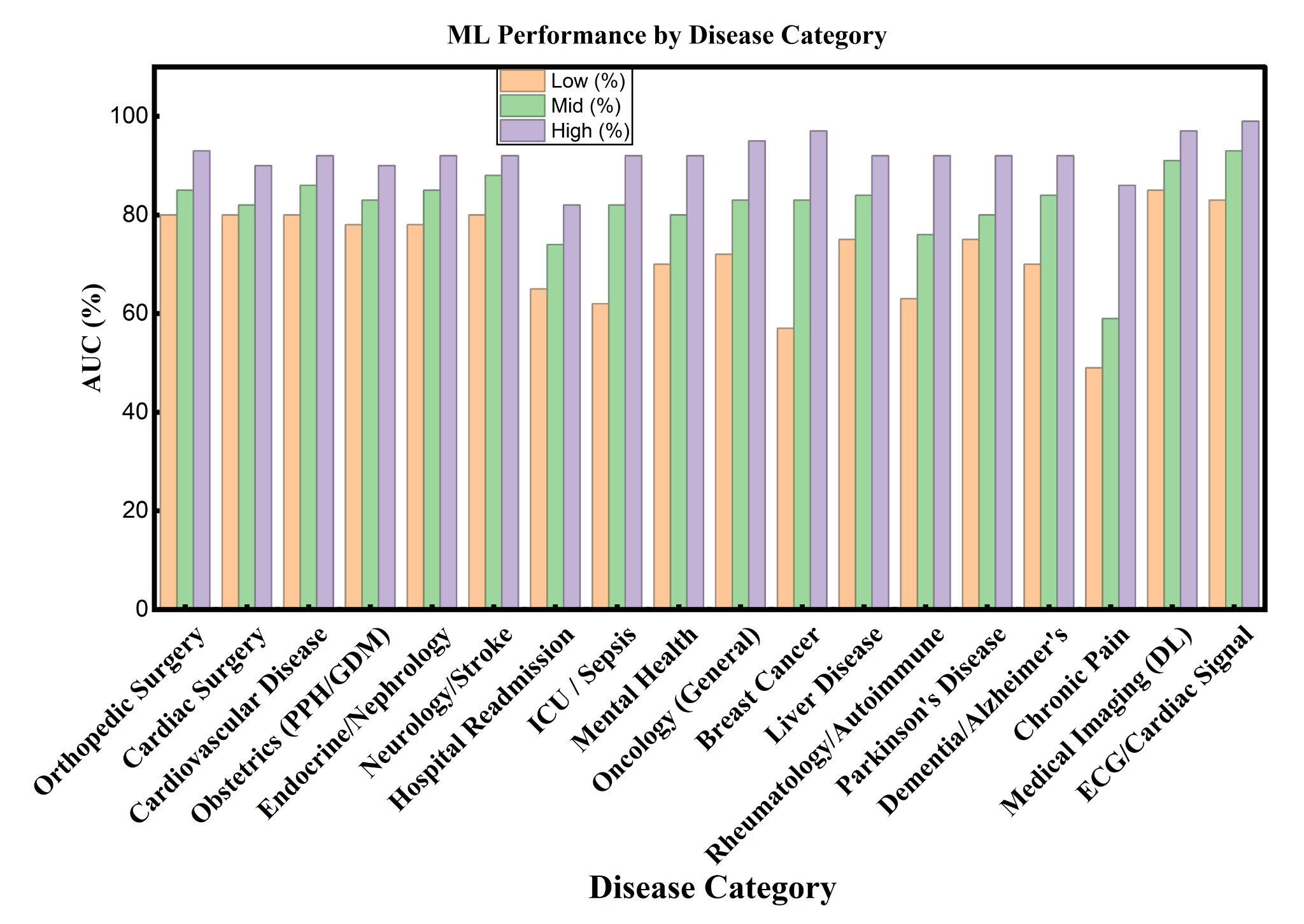}
\caption{Reported low, midpoint, and high performance by disease category across the 104 extracted task-level observations. The plot is descriptive: source studies report heterogeneous metrics and validation designs, so the bars are not pooled estimates of a common effect.}
\label{fig:sup-disease-performance}
\end{figure*}

Figure~\ref{fig:sup-disease-performance} visualizes the category ranges in Table~\ref{tab:disease-summary}. A recurrent middle region is visible across orthopedic surgery, cardiovascular disease, obstetrics, endocrine/renal disease, stroke, mental health, and several oncology tasks, despite substantial differences in pathophysiology and study population. The observation motivates a channel-level explanation: once the recorded variables contain a limited amount of class separation, changing the learner may improve approximation but cannot create missing clinical information.

The figure also shows important departures from the middle region. Chronic pain and hospital readmission include the lowest reported results, whereas imaging and ECG/cardiac-signal studies extend to substantially higher values. These contrasts are consistent with differences in measurement channels, but they do not identify a numeric noise fraction or prove that one modality is universally superior. Chronic pain, for example, is affected by subjective symptom perception, recall, mood, and social context; the observed AUC range of $0.49$--$0.65$ in the cited study \cite{zmudzki2023machine} is therefore consistent with a weak observed channel, but the plateau alone cannot identify the underlying contamination rate. Conversely, high imaging or ECG performance may reflect richer signal, narrower tasks, different validation designs, or some combination of these factors.

Several categories also have broad internal ranges. ICU/sepsis extends from $0.75$ to $0.99$, breast cancer from $0.57$ to $0.97$, and autoimmune/rheumatology from $0.63$ to $0.92$. Such dispersion cautions against treating the disease name as the channel: the actual inputs, outcome definition, cohort construction, and validation protocol determine the frontier. The real-cohort audits in Sections~\ref{sec:real-cohorts}--\ref{sec:learner-panels} address this limitation by measuring frontiers within fixed datasets and metrics.

\subsection{Model-Family Performance}
\begin{figure*}[t]
\centering
\includegraphics[width=.9\textwidth]{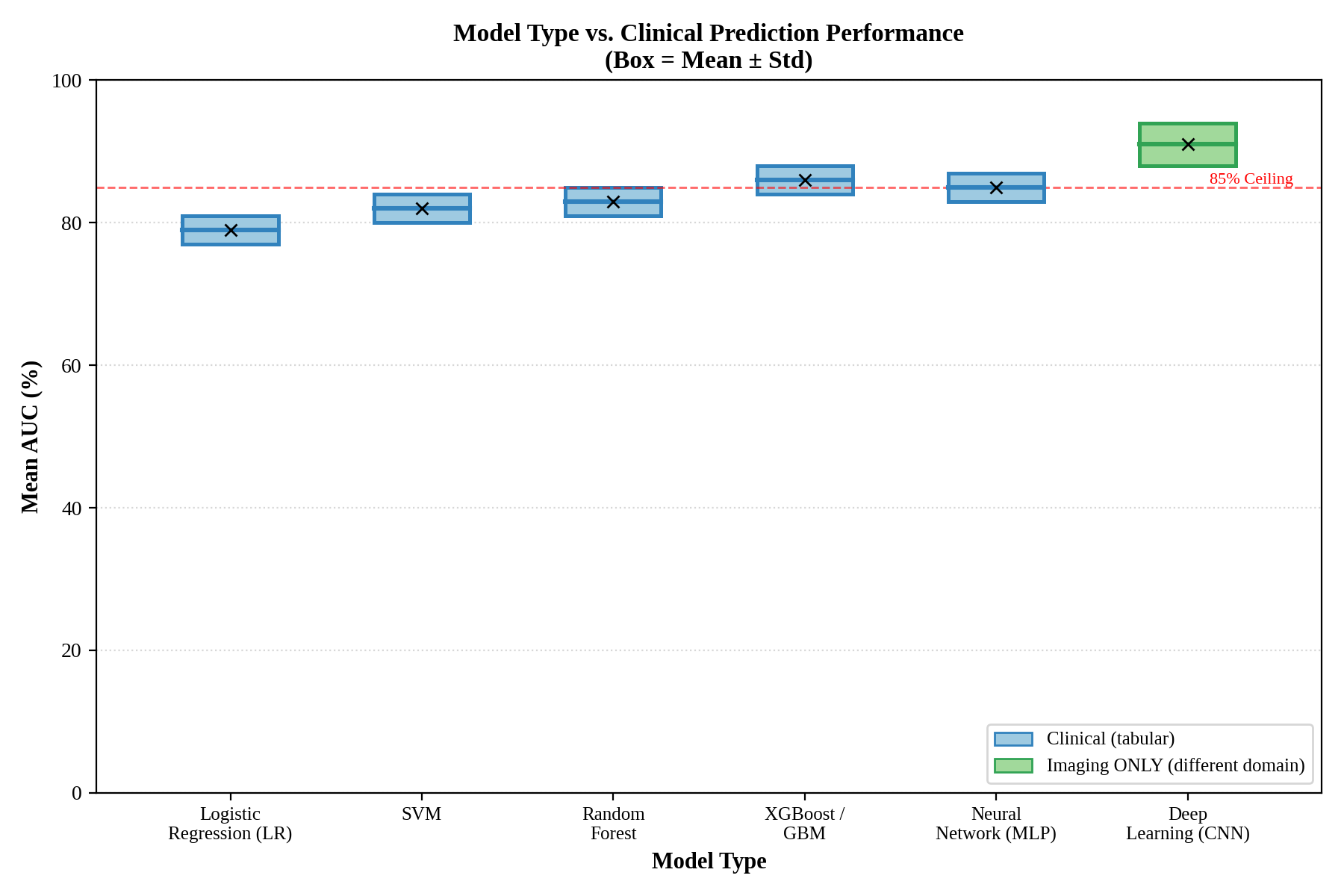}
\caption{Descriptive performance summaries by model family. Clinical tabular models show diminishing gains from logistic regression through boosting and neural networks; the imaging-only CNN bar is a distinct measurement domain and is shown as a comparator rather than evidence that architecture alone moves a fixed-channel frontier.}
\label{fig:sup-model-family}
\end{figure*}

Figure~\ref{fig:sup-model-family} summarizes the reported model-family pattern in the source corpus. Logistic regression typically lies near $0.78$--$0.80$, SVM near $0.80$--$0.82$, random forests near $0.81$--$0.84$, and XGBoost/gradient boosting near $0.83$--$0.87$. Multilayer perceptrons and tabular deep-learning systems generally add little beyond boosting, with typical reported values around $0.84$--$0.88$. The descriptive progression is therefore compatible with diminishing approximation gains: moving from a linear rule to a flexible nonlinear learner can matter, but increasingly complex models often approach the same task-specific information limit.

Cross-study comparisons cannot isolate architecture because disease, sample size, feature set, metric, and validation design all change simultaneously. The learner panels in Tables~\ref{tab:uci-learners}--\ref{tab:nhanes-learners} are therefore the stronger architecture test. Within fixed cohorts, estimated frontiers remain comparatively stable while achieved BA can differ dramatically. In BRFSS and NHANES, for example, MLP AUROC remains competitive while thresholded BA collapses, showing that model complexity can preserve ranking yet fail to realize a useful decision rule. The original complexity-score scatter is omitted because its $r=0.75$ and $R^2=0.75$ annotations are mutually incompatible unless independently verified.

\subsection{Dataset Size and Channel Richness}
\begin{figure*}[t]
\centering
\includegraphics[width=.9\textwidth]{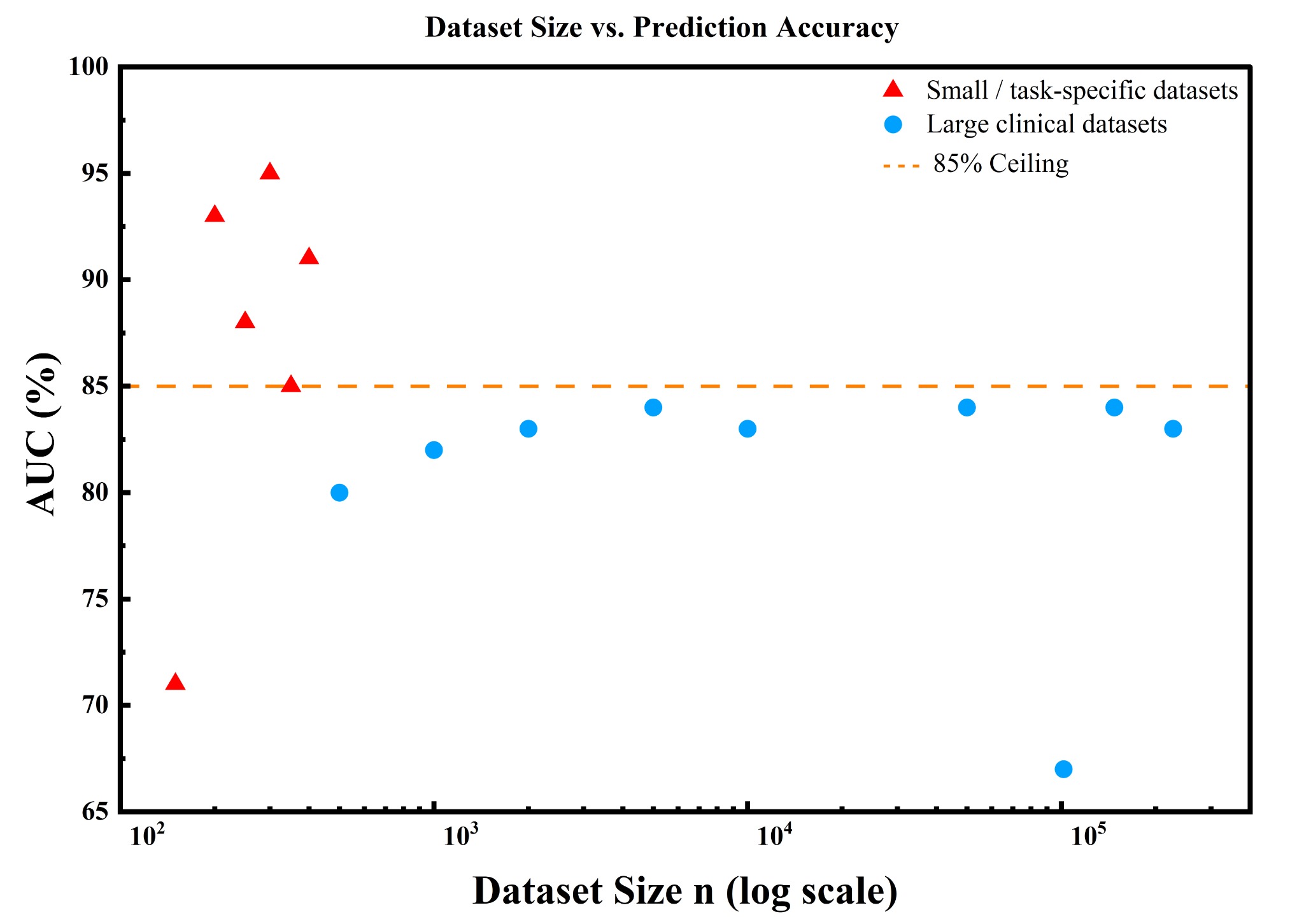}
\caption{Reported performance versus dataset size on a logarithmic scale. Large cohorts do not uniformly exceed the recurrent clinical range, while small task-specific datasets show substantial variance. The plot is descriptive across heterogeneous tasks and metrics.}
\label{fig:sup-dataset-size}
\end{figure*}

Figure~\ref{fig:sup-dataset-size} addresses whether sample size alone breaks the apparent saturation pattern. The large cardiac-surgery study with $227{,}087$ patients reports AUC $0.833$--$0.834$ \cite{sinha2023comparison}, illustrating that a very large cohort can still remain in the recurrent clinical region. At the same time, several small or narrowly defined datasets report much higher values, which may reflect genuinely easier tasks, richer channels, or optimistic validation. The scatter therefore does not support a simple monotone relation between $n$ and reported performance.

The dementia evidence requires particular qualification. The Veronese review should not be treated as a single large-$n$ structured-clinical point: 12 of its 21 studies include CT or MRI, four PET, seven CSF biomarkers, and five blood biomarkers; eight AUCs exceed $0.90$ and five are below $0.80$ \cite{veronese2025clinical}. Its mean therefore averages a heterogeneous, majority-multimodal literature rather than demonstrating a tight structured-record plateau. More informative is the within-outcome channel contrast: the claims-only Reinke cohort has $n=117{,}895$ and modest discrimination, whereas smaller memory-clinic and neuroimaging cohorts often exceed $0.85$--$0.90$ \cite{reinke2023dementia}. Performance thus runs opposite to sample size when the smaller cohorts contain richer measurements, consistent with the distinction between estimation error and channel information.

The high variance among small datasets also highlights publication and validation risks. Small samples can produce unstable estimates, broad uncertainty, and overoptimistic internal validation. Accordingly, Figure~\ref{fig:sup-dataset-size} should be interpreted as evidence that data quantity is not sufficient, not as evidence that data quantity is irrelevant.

\subsection{Clinical-Only and Multimodal Comparisons}
\begin{figure*}[t]
\centering
\includegraphics[width=.9\textwidth]{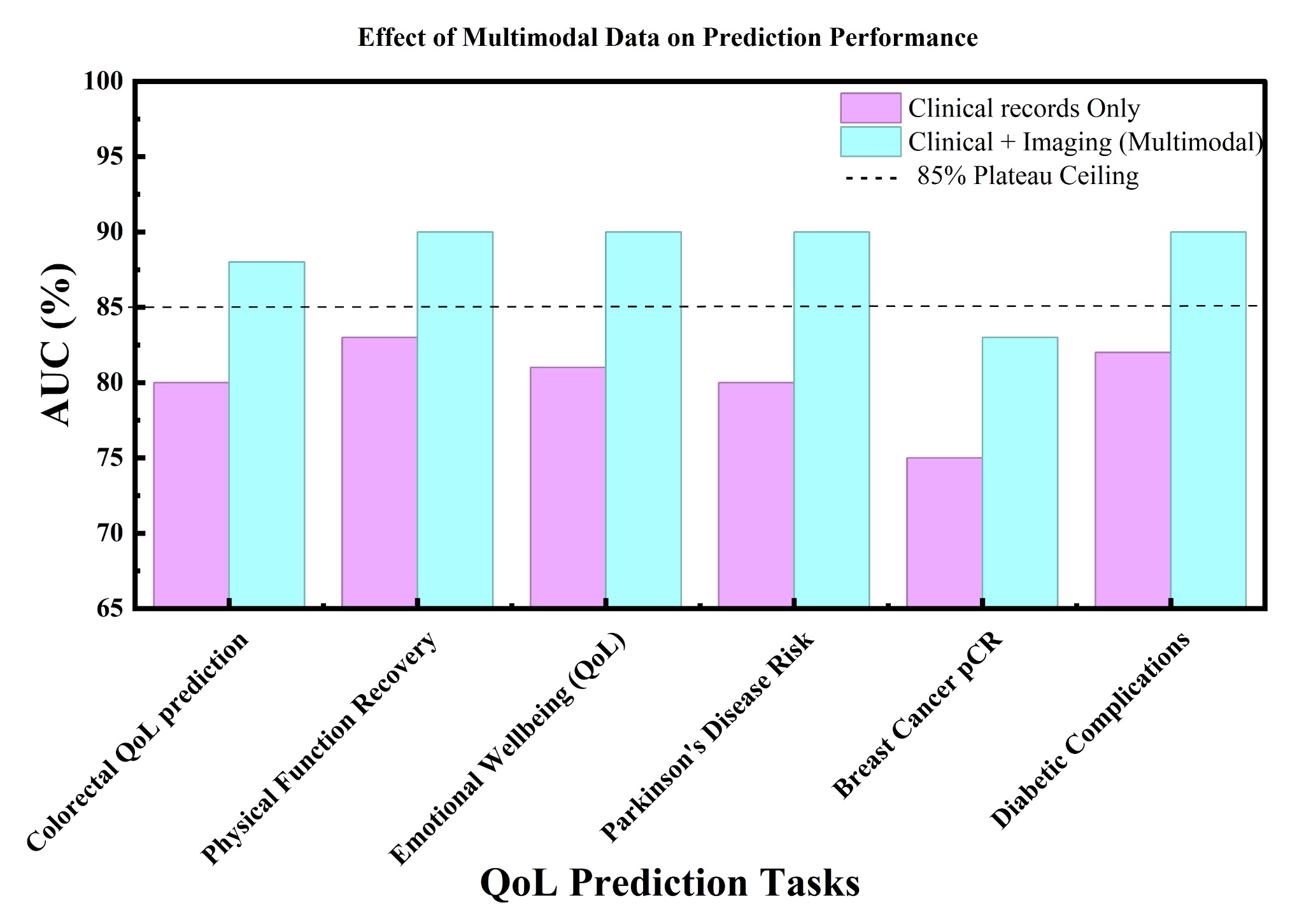}
\caption{Selected clinical-only and multimodal comparisons across quality-of-life and disease-prediction tasks. The figure reports descriptive study-level contrasts rather than a pooled causal effect of adding a modality.}
\label{fig:sup-multimodal}
\end{figure*}

Figure~\ref{fig:sup-multimodal} collects six selected comparisons spanning colorectal quality of life, physical-function recovery, emotional wellbeing, Parkinson disease, breast-cancer response, and diabetic complications. In these examples, clinical-only results lie roughly between $0.75$ and $0.83$, whereas the corresponding multimodal results lie roughly between $0.83$ and $0.90$. The selected contrasts therefore illustrate how adding a channel can move the attainable frontier when it contributes decision-relevant information not already contained in the clinical record.

The Parkinson example extends the observation beyond imaging: a genomics-augmented GenoML system reports AUC $0.897$, compared with approximately $0.78$--$0.83$ for clinical-only models in the same broad domain \cite{makarious2022multimodality}. Nevertheless, these literature comparisons are neither randomized modality ablations nor harmonized within-cohort evaluations. Differences in cohort, outcome, model, and validation can contribute to the apparent gain, so the figure does not establish a universal $5$--$10$ percentage-point effect for ``objective'' data.

The real NHANES experiment provides a more controlled refinement. Questionnaire and measured marginal frontiers are statistically indistinguishable, yet their joint frontier rises by $+0.0471$ over the better marginal. Thus, the useful principle is not that measured data are intrinsically cleaner than questionnaires; it is that a new channel helps when it changes decisions on a positive-probability subset, as characterized by Theorem~4. Figure~\ref{fig:sup-multimodal} is therefore best read as descriptive evidence of possible complementarity, while the cohort audits supply the direct within-dataset test.

\subsection{Empirical Saturation Formula}
\label{sec:empirical-formula}
The source evidence synthesis proposed an exploratory pre-training benchmark linking reported AUROC to sample size $n$ and a coarse model-complexity score $C$:
\begin{equation}
\widehat{\mathrm{AUC}}(n,C)
=A_{\min}+(A_{\max}-A_{\min})
\left[1-\exp\left\{-k\ln\left(\frac{n}{n_0}\right)\right\}\right]
+\beta C.
\label{eq:empirical-saturation}
\end{equation}
The supplied parameterization is $A_{\min}=0.68$, $A_{\max}=0.875$, $k=0.18$, $n_0=500$, and $\beta=0.012$. The complexity score assigns $C=0$ to logistic regression, $1$ to SVM, $2$ to random forest, $3$ to XGBoost/gradient boosting, $4$ to an MLP, and $5$ to tabular deep learning. Because the review eligibility criterion emphasizes studies with $n\ge n_0$, the expression is used only over that range. At $n=n_0$, the sample-size term is zero; as $n$ grows, its derivative decreases and the curve approaches $A_{\max}+\beta C$. Under the supplied coding, the largest asymptote is $0.875+5(0.012)=0.935$.

Equation~\ref{eq:empirical-saturation} encodes two observations from the descriptive corpus: rapid initial benefit from additional data followed by diminishing returns, and a smaller additive gain with model complexity. It is not a theorem and does not follow from the total-variation frontier. In particular, the $\beta C$ term permits different model families to approach different asymptotes, whereas Lemma~2 gives a common population frontier for all learners using the same variables. The formula should therefore be interpreted as a phenomenological summary of heterogeneous published studies, not as an estimator of a cohort-specific Bayes ceiling.

\begin{table*}[t]
\centering
\scriptsize
\caption{Selected validation cases reported for the empirical saturation formula. Errors are predicted minus reported AUROC.}
\label{tab:formula-validation}
\begin{tabular}{lrrrrr}
\toprule
Study & $n$ & $C$ & Reported & Predicted & Error\\
\midrule
Cardiac surgery \cite{sinha2023comparison} & $227{,}087$ & 3 & $0.834$ & $0.856$ & $+0.022$\\
Dementia review \cite{veronese2025clinical} & $>1{,}000{,}000$ & 2 & $0.845$ & $0.845$ & $0.000$\\
Orthopedic review \cite{ogink2021wide} & $5{,}507$ & 2 & $0.800$ & $0.815$ & $+0.015$\\
Postpartum hemorrhage \cite{ranjbar2023predicting} & $\sim5{,}000$ & 3 & $0.850$ & $0.840$ & $-0.010$\\
Stroke review \cite{yang2023predictive} & $\sim3{,}000$ & 3 & $0.872$ & $0.833$ & $-0.039$\\
30-day readmission \cite{emijohnson2025predicting} & $101{,}766$ & 3 & $0.667$ & $0.860$ & $+0.193$\\
CVD EHR review \cite{liu2025machine} & $\sim50{,}000$ & 2 & $0.865$ & $0.845$ & $-0.020$\\
GPT-4 perioperative \cite{chung2024large} & $\sim1{,}000$ & 4 & $0.810$ & $0.759$ & $-0.051$\\
\bottomrule
\end{tabular}
\end{table*}

The tabulated values show where the heuristic succeeds and where it fails. Cardiac surgery, dementia, orthopedic prediction, postpartum hemorrhage, and CVD EHR fall within $\pm0.03$ of the formula, while stroke, GPT-4 perioperative prediction, and especially 30-day readmission depart more substantially. The source narrative states that six of eight cases are within $\pm0.03$, but the supplied table yields five; we retain the tabulated values and do not repeat the inconsistent count. The readmission error is particularly informative in light of the real-cohort audit: GBDT nearly reaches the measured frontier on the UCI cohort, so the low reported AUROC is better explained by the recorded channel than by insufficient algorithmic complexity. A formula depending only on $n$ and $C$ cannot represent this channel-specific limitation.

\begin{figure}[t]
\centering
\begin{minipage}[t]{.49\textwidth}
\includegraphics[width=\linewidth]{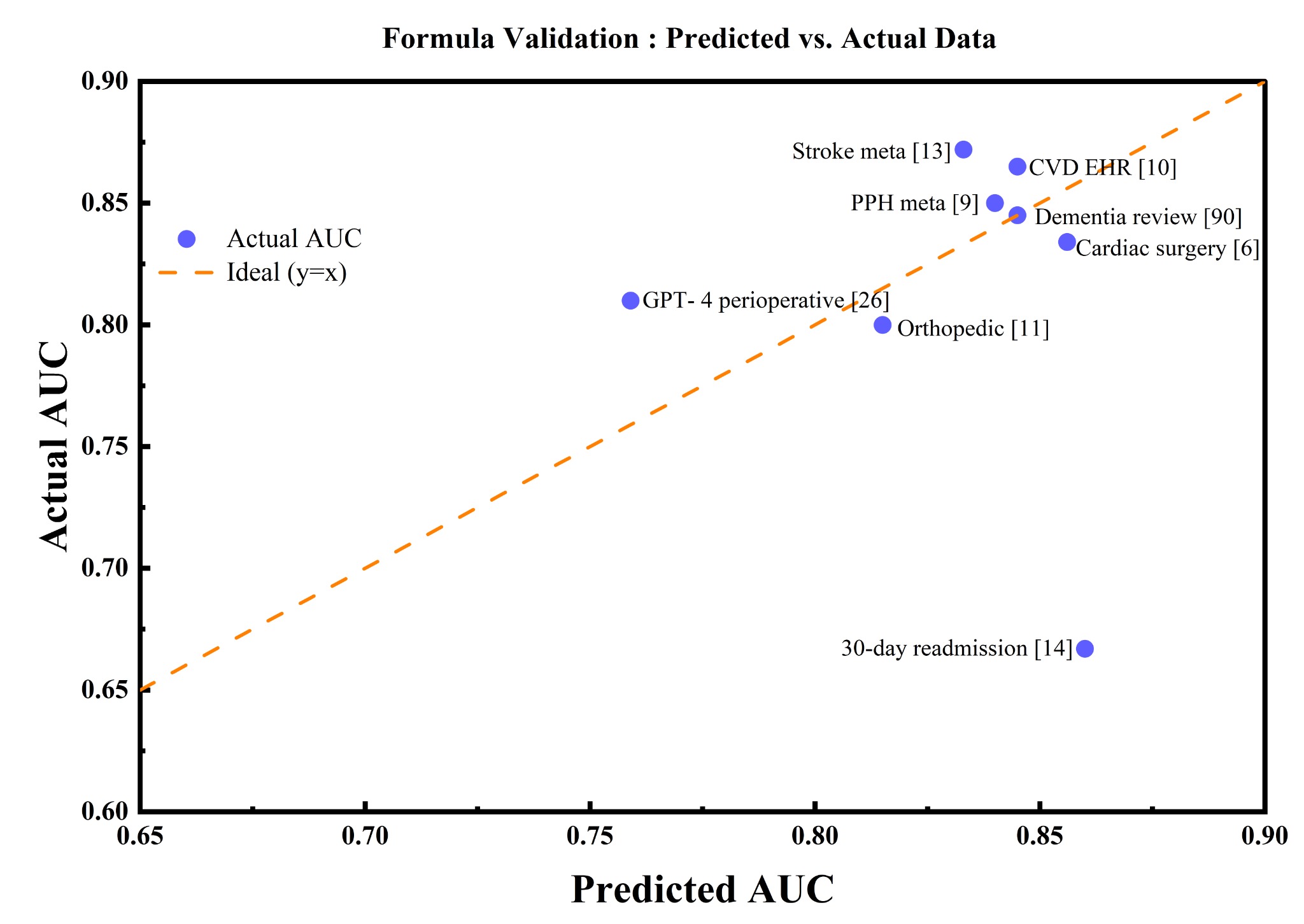}
\end{minipage}\hfill
\begin{minipage}[t]{.49\textwidth}
\includegraphics[width=\linewidth]{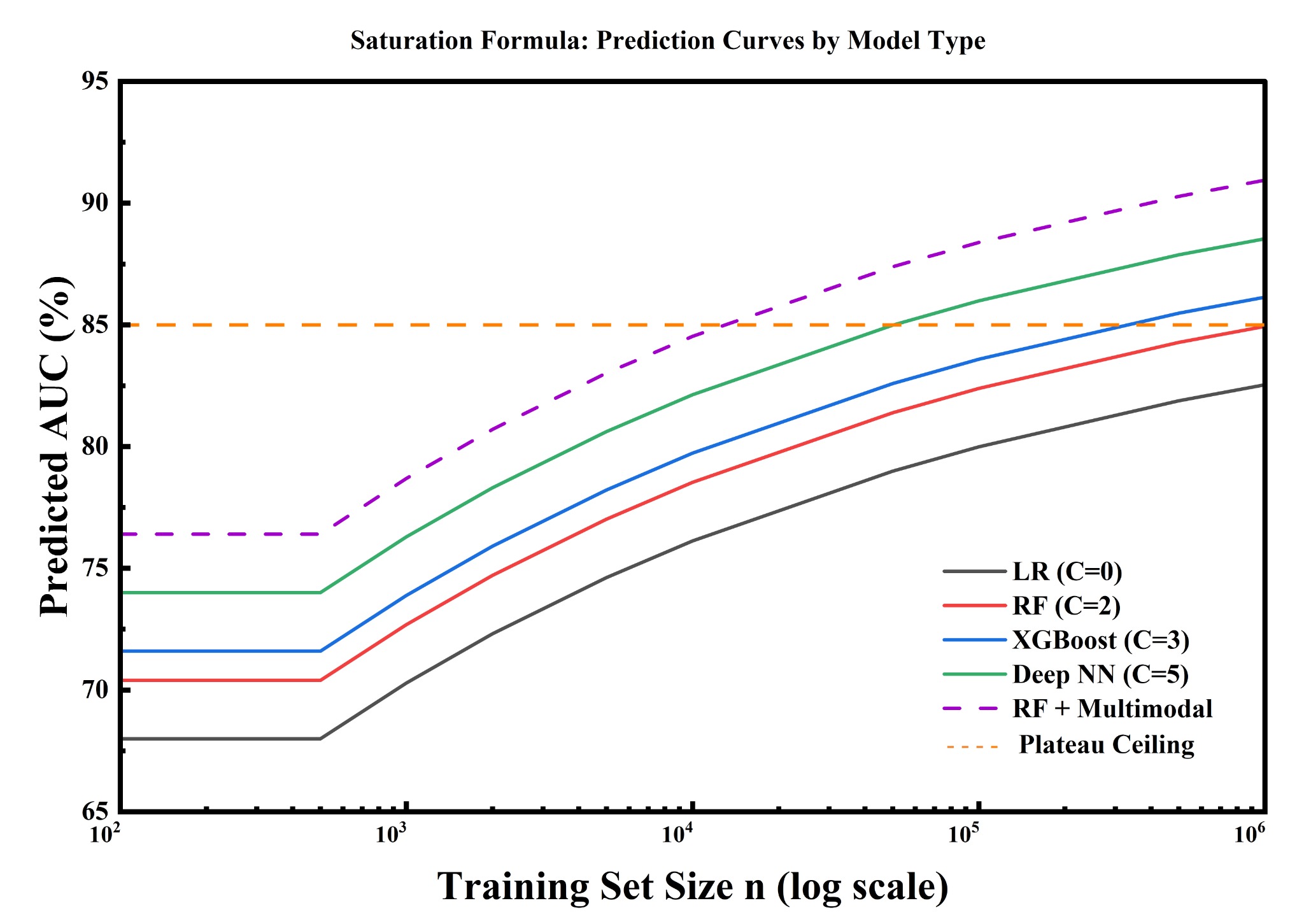}
\end{minipage}
\caption{Exploratory empirical saturation formula. Left: predicted versus reported AUROC for selected studies. Right: prediction curves as functions of training-set size and model-complexity score. The curves summarize the aggregate corpus and are not estimates of a common fixed-channel frontier.}
\label{fig:sup-empirical-formula}
\end{figure}

The left panel of Figure~\ref{fig:sup-empirical-formula} visualizes the validation cases in Table~\ref{tab:formula-validation}; distance from the diagonal exposes channel- or task-specific departures that $n$ and $C$ cannot capture. The right panel illustrates the intended diminishing-return behavior: the curves rise quickly at smaller $n$ and flatten as sample size increases, while larger complexity scores shift the predicted AUROC upward. The multimodal curve is shown separately because the original analysis treated additional measurement channels as a change in the attainable regime rather than merely another complexity increment.

As a practical heuristic, Equation~\ref{eq:empirical-saturation} can warn against expecting very high AUROC solely from a larger tabular cohort or a more complex learner. It should not be used for formal sample-size determination, channel-frontier estimation, or claims of a universal $0.875$ ceiling. Its parameters were fitted to aggregate summaries, uncertainty was not propagated, overlap among reviews was not modeled, and performance metrics and validation designs vary across sources. The cross-fitted frontier audit developed in the main paper supersedes the formula as the operational method because it estimates the information available in a specified cohort and directly separates learner gap from measurement-channel limitation. The superseded information-theoretic bottleneck diagram is omitted because positive conditional mutual information alone does not guarantee a change in hard decisions.

\bibliography{aaai2027}

\end{document}